\documentclass{article}

\usepackage[preprint]{neurips_2026}

\usepackage[utf8]{inputenc} 
\usepackage[T1]{fontenc}    
\usepackage{hyperref}       
\usepackage{url}            
\usepackage{booktabs}       
\usepackage{amsfonts}       
\usepackage{nicefrac}       
\usepackage{microtype}      
\usepackage{xcolor}         
\usepackage{amsmath}
\usepackage{amsthm}
\DeclareMathOperator{\vecop}{vec}
\usepackage{graphicx}
\usepackage{subcaption}
\newtheorem{theorem}{Theorem}
\newtheorem{lemma}{Lemma}
\title{Spectral Flattening Is All Muon Needs: How Orthogonalization Controls Learning Rate and Convergence}

\author{%
  Tien-Phat Nguyen\thanks{Equal contribution.} \\
  Hanoi University of \\
  Science and Technology \\
  Hanoi, Vietnam \\
  \texttt{tien.phat140205@gmail.com} \\
  \And
  Truong Nguyen\footnotemark[1] \\
  Hanoi University of \\
  Science and Technology \\
  Hanoi, Vietnam \\
  \texttt{tonytruong23305@gmail.com} \\
  \And
  Minh-Phuc Truong\footnotemark[1] \\
  Hanoi University of \\
  Science and Technology \\
  Hanoi, Vietnam \\
  \texttt{truongminhphuc08102005@gmail.com} \\
  \And
  Tuc Nguyen \\
  Indiana University \\
  107 S. Indiana Ave, Bloomington, \\
  IN 47405, USA \\
  \texttt{tucnguye@iu.edu} \\
  \And
  James Bailey \\
  Monash University \\
  Clayton, VIC 3800, Australia \\
  \texttt{baileyj@unimelb.edu.au} \\
  \And
  Trung Le \\
  Monash University \\
  Clayton, VIC 3800, Australia \\
  \texttt{trunglm@monash.edu} \\
}

\begin{document}

\maketitle

\begin{abstract}
Muon orthogonalizes the momentum buffer before each update, replacing its singular values with ones via Newton--Schulz iterations. This simple change lets Muon tolerate far larger learning rates and converge faster than other optimizers---but \emph{why}? We show that the mechanism is \emph{spectral flattening}, and develop two results around it. First, we prove that Muon's maximal stable step size scales with the \emph{average} singular value of the gradient rather than the \emph{largest}, which bottlenecks standard gradient descent. Second, we recast Muon as a preconditioned gradient method and show, under a Kronecker-factored curvature model, that it improves the effective convergence factor, with the improvement controlled by the spectrum of the gradient covariance. Extensive experiments validate both results: Muon remains stable at learning rates that cause SGD to diverge within the first few iterations, and reaches accuracy milestones several epochs earlier even at identical step sizes. Taken together, our results offer a principled, geometric explanation for Muon's empirical success.
\end{abstract}

\section{Introduction}

Training modern neural networks is fundamentally an optimization problem over a highly nonconvex landscape. As models grow deeper and wider, the loss surface contains many flat regions, sharp directions, saddle points, and local structures that can slow or destabilize training~\citep{liu2022loss}. Consequently, the choice of optimizer is not a minor implementation detail: it often determines whether a model can be trained efficiently at all. This has motivated a long line of practical optimizers, from Adam~\citep{kingma2015adam} to AdamW~\citep{loshchilov2019decoupled}, that improve stability and convergence by changing how gradients are scaled, accumulated, or regularized.

Recently, Muon has emerged as a particularly striking alternative~\citep{jordan2024muon}. Unlike coordinate-wise adaptive methods such as AdamW, Muon treats each weight matrix as a matrix: before applying an update, it uses a few Newton--Schulz iterations to transform the gradient into an approximate polar factor $UV^\top$. This operation preserves the singular vectors of the update while flattening its singular values. Empirically, this matrix-level normalization has made Muon competitive in important deep learning settings. It has been scaled to large language model training with reported gains over AdamW~\citep{liu2025muon}, has accelerated grokking in transformer experiments~\citep{tveit2025grokking}, and has appeared among the fastest optimizers in recent systematic studies of language-model pretraining~\citep{wen2026fantastic}. These applications make Muon an important object of study, not merely a heuristic variant of gradient descent. 
However, the same empirical success also raises a basic theoretical question: \emph{why does Muon work?} A common intuition is that orthogonalization balances the update by removing the influence of very large singular values. This suggests that Muon should be less sensitive to dominant gradient directions, should tolerate larger learning rates, and should move more efficiently through anisotropic curvature. Existing theory has begun to clarify parts of this picture. For example, Muon can be interpreted as steepest descent under a specific matrix norm~\citep{bernstein2024old}, and recent empirical studies have examined when its speedups persist under careful tuning~\citep{wen2026fantastic}. Yet these results do not fully answer the quantitative questions most relevant to optimization: \emph{how much} larger can Muon's learning rate be, and \emph{why} should its orthogonalized update converge faster?

This paper addresses this gap. Our central claim is that Muon's advantage comes from \emph{spectral flattening}: by equalizing the singular values of the update gradient, Muon prevents a single large singular direction from controlling the step-size constraint, while also acting as a one-sided preconditioner for matrix-valued parameters. We develop this claim through three contributions:

1. \textbf{Maximal learning rate.} We derive exact one-step descent thresholds for both gradient descent (SGD) and Muon. The comparison is sharp and quantitative: Muon's maximal stable learning rate is $\frac{2}{\lambda_{\max}^H}\frac{\sum_{i=1}^m \sigma_i}{m}$, governed by the \emph{average} singular value of the gradient, whereas SGD's is $\frac{2}{\lambda_{\max}^H}$, bottlenecked by the \emph{largest} singular value through the Hessian. Under a Gauss--Newton/K-FAC Hessian approximation, the gap is controlled by the ratio $\frac{\sum \sigma_i}{\sigma_{\max}}$, which grows with gradient spectral concentration. This provides a direct, mechanistic explanation for why Muon tolerates substantially larger step sizes: spectral flattening prevents any single singular direction from dominating the descent condition.

2. \textbf{Convergence rate.} We recast Muon as a preconditioned gradient method with preconditioner $\boldsymbol{P} = \mathbb{I}_n \otimes (GG^\top)^{-1/2}$ and analyze it under relative smoothness and Polyak--\L{}ojasiewicz conditions. This reveals a structural acceleration: GD converges with factor $1-\alpha/\beta$, while Muon converges with the improved factor $1-\tilde{\alpha}/\tilde{\beta}$, where $\tilde{\alpha} = \lambda_{\min}(\boldsymbol{P}H)$ and $\tilde{\beta} = \lambda_{\max}(\boldsymbol{P}H)$. Under a Kronecker-factored curvature model, we prove that $\frac{\tilde{\alpha}}{\tilde{\beta}} = \frac{\alpha}{\beta}\big/\sqrt{\frac{\lambda_{\min}(G_t G_t^\top)}{\lambda_{\max}(G_t G_t^\top)}} > \frac{\alpha}{\beta}$ whenever the gradient covariance is anisotropic, yielding a strictly faster linear rate. The improvement is directly tied to the spectral spread of the gradient: the more ill-conditioned the gradient covariance, the larger the gap.

3. \textbf{Experimental validation.} We validate both results on CIFAR-10 with CifarNet using a dual-optimizer strategy that isolates each optimizer's effect on convolutional layers. Learning rate sweeps confirm that Muon remains stable where SGD diverges, and convergence experiments at identical learning rates show Muon reaching accuracy milestones earlier with a consistently lower per-step convergence ratio $r_t$. Our analysis also yields a new perspective on normalization: because the maximal stable learning rate scales inversely with $\lambda_{\max}(\boldsymbol{X}^\top\boldsymbol{X})$, we propose and verify a normalization principle that expands the usable learning rate range for both optimizers.

\section{Related Work}
\textbf{Optimizers for deep learning.}
The practical success of deep learning is closely tied to stable and efficient first-order optimization. Adam~\citep{kingma2015adam} and AdamW~\citep{loshchilov2019decoupled} improve robustness through momentum, coordinate-wise normalization, and decoupled regularization, although their benefits can depend strongly on task and tuning protocol~\citep{wilson2017marginal}. Our work follows the broader effort to understand optimizer-induced training dynamics, but focuses on a matrix-valued method whose normalization acts on singular directions rather than coordinates.

\textbf{Understanding Muon.}
Muon was introduced as a first-order optimizer that approximates the polar factor of each gradient matrix using Newton--Schulz iterations~\citep{jordan2024muon}. Recent empirical work has shown that Muon can scale to language-model pretraining and related settings~\citep{liu2025muon,shah2025practical}, with reported applications to grokking, latent-attention and MoE transformers, finetuning, and quantized optimizer states~\citep{tveit2025grokking,mehta2025muon,page2025muonall,gupta2025quantizing}. Systematic benchmarking further places Muon among the fastest matrix-based optimizers, while showing that its gains depend on scale and tuning protocol~\citep{wen2026fantastic}. This empirical progress has also motivated variants such as blockwise Muon, AdaMuon, MuonMax, and NAMO~\citep{boreiko2025towards,si2025adamuon,crawshaw2025exploration,zhang2026adamimproves}.

Theoretical work has started to explain Muon's geometry. Existing analyses connect Muon to steepest descent under matrix norms~\citep{bernstein2024old,li2025note}, non-Euclidean trust-region optimization~\citep{kovalev2025understanding}, spectral-norm constrained optimization~\citep{chen2025spectralconstraints}, and conditions under which spectral updates outperform Euclidean ones~\citep{davis2025spectralupdates}. These works clarify important aspects of the update, but they do not directly quantify how orthogonalization changes the maximal stable learning rate or the effective convergence factor. Our paper fills this gap by showing that spectral flattening changes the learning-rate bottleneck from the largest singular value to an average singular scale, and improves the preconditioned convergence factor under a Kronecker-factored curvature model.

\section{Preliminaries}
\subsection{Muon Optimizer}
Muon~\citep{jordan2024muon} is a matrix-valued optimizer. At iteration $t$, given the current weight $\boldsymbol{W}_{t-1}$, we update as follows:
\begin{align}
\boldsymbol{G}_{t} & =\nabla\mathcal{L}\left(\boldsymbol{W}_{t-1}\right)\nonumber \\
\boldsymbol{O}_{t} & =\text{Newton-Schulz}\left(\boldsymbol{G}_t\right)\nonumber \\
\boldsymbol{W}_{t} & =\boldsymbol{W}_{t-1}-\eta\boldsymbol{O}_t,\label{eq:Muon}
\end{align}
where $\eta$ is a learning rate. Here we note that Newton-Schulz is used to iteratively approximate $\left(\boldsymbol{G}_{t}\boldsymbol{G}_{t}^{\top}\right)^{-1/2}\boldsymbol{G}_{t}$. Specifically, let $\boldsymbol{G}_{t}=U\Sigma V^\top$ be the SVD of $\boldsymbol{G}_t$, we then have $\left(\boldsymbol{G}_{t}\boldsymbol{G}_{t}^{\top}\right)^{-1/2}\boldsymbol{G}_{t}=UV^\top$. Geometrically, $UV^\top$ is the closest orthogonal matrix to $\boldsymbol{G}_t$ in Frobenius norm, effectively replacing the singular values with ones while preserving the singular directions. 

\subsection{Stochastic Gradient Descent (SGD) Optimizer}
At iteration $t$, given the current weight $\boldsymbol{W}_{t-1}$, we update as follows:
\begin{align}
\boldsymbol{G}_{t} & =\nabla\mathcal{L}\left(\boldsymbol{W}_{t-1}\right)\nonumber \\
\boldsymbol{W}_{t} & =\boldsymbol{W}_{t-1}-\eta\boldsymbol{G}_t,\label{eq:GD}
\end{align}
where $\eta$ is a learning rate.

\subsection{K-FAC Hessian Approximation}\label{sec:kfac}
Our learning-rate comparison uses a local Gauss--Newton/K-FAC approximation of the layerwise Hessian~\citep{martens2015optimizing}. For a weight matrix with input activations $\boldsymbol{X}$ and gradient matrix $G_t$, this approximation factorizes the curvature as
\begin{equation}
H \approx \boldsymbol{X}^{\top}\boldsymbol{X}\otimes G_tG_t^\top.
\end{equation}
This Kronecker structure lets us express the dominant curvature scale through the largest eigenvalues of the input covariance and gradient covariance, which is the form used in Section~\ref{sec:section_2}.



\section{Theoretical Analysis}\label{sec:theory}
\subsection{Analysis on Learning Rate}
\label{sec:section_2}

First, we compare the maximal learning rate of SGD and Muon. By a second-order Taylor expansion, we have the quadratic upper bound
\begin{equation}
\mathcal{L}\left(\boldsymbol{W}_{t-1}+\Delta\boldsymbol{W}\right)\leq\mathcal{L}\left(\boldsymbol{W}_{t-1}\right)+\nabla\mathcal{L}\left(\boldsymbol{W}_{t-1}\right)^{\top}\Delta\boldsymbol{W}+\frac{\lambda_{\text{max}}^{H}}{2}\Vert\Delta\boldsymbol{W}\Vert_{2}^{2}.\label{eq:Taylor_expansion}
\end{equation}
where $\lambda_{\text{max}}^{H}:=\lambda_{\text{max}}(H)$ denotes the largest eigenvalue of the Hessian. Detailed derivation could be found in Appendix~\ref{app:taylor}

\subsubsection{Theoretical Analysis for Stochastic Gradient Descent} 

For SGD optimizer, we have $\Delta\boldsymbol{W}=-\eta\boldsymbol{G}_{t}$, leading to
\begin{equation}
\mathcal{L}\left(\boldsymbol{W}_{t-1}+\Delta\boldsymbol{W}\right)\leq\mathcal{L}\left(\boldsymbol{W}_{t-1}\right)-\eta\nabla\mathcal{L}\left(\boldsymbol{W}_{t-1}\right)^{\top}\boldsymbol{G}_{t}+\frac{\lambda_{\text{max}}^{H}}{2}\eta^{2}\Vert\boldsymbol{G}_{t}\Vert_{2}^{2}.\label{eq:gd_bound}
\end{equation}
The update is efficient (i.e., $\mathcal{L}\left(\boldsymbol{W}_{t}\right)<\mathcal{L}\left(\boldsymbol{W}_{t-1}\right)$) if we have
\begin{equation}
\eta_{t}<\eta_{t}^{\max}:=\frac{2}{\lambda_{\text{max}}^{H}}\frac{\nabla\mathcal{L}\left(\boldsymbol{W}_{t-1}\right)^{\top}\boldsymbol{G}_{t}}{\Vert\boldsymbol{G}_{t}\Vert_{2}^{2}}.\label{eq:gd_eta_max}
\end{equation}
\begin{theorem}
    \label{thm:gd_lr}
    For gradient descent with $\boldsymbol{G}_{t}=\nabla\mathcal{L}(\boldsymbol{W}_{t-1})$, the maximal learning rate in~\eqref{eq:gd_eta_max} reduces to $\eta_{t}^{\mathrm{max}}=\frac{2}{\lambda_{\mathrm{max}}^{H}}$.


\end{theorem}
The proof is deferred to Appendix~\ref{sec:proof_gd_lr}.

\subsubsection{Theoretical Analysis for Muon} 
\label{sec:lr_muon}

For the Muon optimizer, we have $\Delta\boldsymbol{W}=-\eta\boldsymbol{O}_{t}$. Applying the quadratic bound~\eqref{eq:Taylor_expansion} and simplifying yields
\begin{equation}
\mathcal{L}\left(\boldsymbol{W}_{t-1}+\Delta\boldsymbol{W}\right)\leq\mathcal{L}\left(\boldsymbol{W}_{t-1}\right)-\eta\,\text{tr}\left(G_{t}^{\top}\boldsymbol{O}_{t}\right)+\frac{\lambda_{\text{max}}^{H}}{2}\eta^{2}m,\label{eq:gap_Muon}
\end{equation}
where $m$ is the number of rows of $\boldsymbol{O}_t \in \mathbb{R}^{m \times n}$, assuming $m \leq n$. A formal derivation could be found in Appendix~\ref{app:muon_bound}

The update is efficient (i.e., $\mathcal{L}\left(\boldsymbol{W}_{t}\right)<\mathcal{L}\left(\boldsymbol{W}_{t-1}\right)$) if we have
\begin{equation}
\eta_{t}\leq\eta_{t}^{\text{max}}:=\frac{2}{\lambda_{\text{max}}^{H}}\frac{\text{tr}\left(G_{t}^{\top}\boldsymbol{O}_{t}\right)}{m}.\label{eq:eta_t_max_Muon}
\end{equation}

\begin{theorem}
    \label{thm:eta_Muon}
    For the Muon optimizer, we then have $\eta_{t}^{\text{max}}=\frac{2}{\lambda_{\text{max}}^{H}}\frac{\sum_{i=1}^{m}\sigma_{i}}{m}$, where $\sigma_{1:m}$ are the singular values of the gradient matrix $G_t$.


\end{theorem}
The proof is deferred to Appendix~\ref{sec:proof_eta_muon}.
To make the maximal learning rate tractable, we adopt the Gauss-Newton/K-FAC approximation for the Hessian (Section~\ref{sec:kfac})
\begin{equation}
H=\boldsymbol{X}^{\top}\boldsymbol{X} \otimes G_{t}G_{t}^\top,\label{eq:Hessian_approx}
\end{equation}
where $\boldsymbol{X} \in \mathbb{R}^{b \times d}$ is the data to this layer and $\otimes$ represents the Kronecker product.

We estimate the maximal eigen-value of the Hessian matrix as
\begin{equation}
\lambda_{\text{max}}^{H}=\lambda_{\text{max}}(\boldsymbol{X}^{\top}\boldsymbol{X})\lambda_{\text{max}}(G_{t}G_{t}^\top)=\lambda_{\text{max}}(\boldsymbol{X}^{\top}\boldsymbol{X})\sigma_{\text{max}}(G_t)^2.\label{eq:eigen_max_estimate}
\end{equation}

We now compare the $\eta_t^{\text{max}}$ for two optimizers:
\begin{itemize}
    \item For SGD, we can approximate $\eta_{t}^{\text{max}}=\frac{2}{\lambda_{\text{max}}^{H}}\approx\frac{2}{\lambda_{\text{max}}(\boldsymbol{X}^{\top}\boldsymbol{X})\sigma_{\text{max}}(G_t)}\cdot\frac{1}{\sigma_{\text{max}}(G_t)}$. This further implies that GD needs to set a very small learning rate if $\sigma_{\text{max}}(G)$ is high.
    \item For Muon, we can approximate $\eta_{t}^{\text{max}}=\frac{2}{\lambda_{\text{max}}^{H}}\frac{\sum_{i=1}^{m}\sigma_{i}}{m}\approx\frac{2}{\lambda_{\text{max}}(\boldsymbol{X}^{\top}\boldsymbol{X})\sigma_{\text{max}}(G_t)}\cdot\frac{\sum_{i=1}^{m}\sigma_{i}}{m\sigma_{\text{max}}(G_t)}$. Compared to GD, Muon haves a second factor $\frac{\sum_{i=1}^{m}\sigma_{i}}{m\sigma_{\text{max}}(G_t)}$, which compensates for a large $\sigma_{\text{max}}(G_t)$ by averaging over all singular values---a direct consequence of spectral flattening.
\end{itemize}

Experimental validation of these learning-rate bounds is provided in Section~\ref{sec:exp_lr}; the role of $\lambda_{\text{max}}(\boldsymbol{X}^{\top}\boldsymbol{X})$ is further examined through normalization layers in Section~\ref{sec:normalization}.


\subsection{Analysis on Convergence Rate}

\subsubsection{Theoretical Analysis for Convergence Rate of Stochastic Gradient Descent}

We now investigate the convergence rate of SGD. Similar to other works, we make the following assumptions:
\begin{itemize}
    \item \textbf{A1:} The loss function is $\beta$-smooth with $\beta>0$ (e.g., $\beta$ could be set to $\lambda_{\text{max}}^H$).
    \item \textbf{A2:} The loss function satisfies the PL condition:
    \begin{equation}
\frac{1}{2}\Vert\nabla\mathcal{L}(\boldsymbol{W})\Vert_{2}^2\geq\alpha\left(\mathcal{L}(\boldsymbol{W})-\mathcal{L}^{*}\right),\label{eq:PL_cond}
\end{equation}
where $\mathcal{L}^*$ is the minimal objective value and $\alpha >0$ (e.g., $\alpha$ could be set to $\lambda_{\text{min}}^H$). 
\end{itemize}

\begin{theorem}
    \label{thm:GD_converenge_rate}
    Assuming the above assumptions and setting the learning rate $\eta_t = \frac{1}{\beta}$, the convergence rate of GD is
    \begin{equation}
\mathcal{L}\left(\boldsymbol{W}_{t}\right)-\mathcal{L}^{*}\leq\left(1-\frac{\alpha}{\beta}\right)^{t}\left(\mathcal{L}\left(\boldsymbol{W}_{0}\right)-\mathcal{L}^{*}\right).\label{eq:conv_rate_GD}
\end{equation}
\end{theorem}

It is evident that the convergence rate of SGD depends on the ratio $\frac{\alpha}{\beta}=\frac{\lambda_{\text{min}}^{H}}{\lambda_{\text{max}}^{H}}$. Higher this ratio is, faster convergence is. The proof is deferred to Appendix~\ref{sec:proof_gd_conv}.

\subsubsection{Theoretical Analysis for Convergence Rate of Muon}

We now rewrite the Muon update in the vectorial form as follows (see Appendix \ref{sec:vector_form} for a detailed derivation):
\[
\boldsymbol{W}_{t}=\boldsymbol{W}_{t-1}-\eta\boldsymbol{P}_{t}\boldsymbol{g}_{t},
\]
where $\boldsymbol{W}_{t-1} \in \mathbb{R}^d$ with $d= m \times n$ is the vector form of $\boldsymbol{W}_{t-1}\in \mathbb{R}^{m \times n}$, $\boldsymbol{P}_{t} = \mathbb{I}_n \otimes (G_{t}G_{t}^\top)^{-\frac{1}{2}} \in \mathbb{R}^{d \times d}$, and $\boldsymbol{g}_{t} = \text{vec}(G_{t})$.    

\begin{lemma}
    \label{lem:smooth}
    With $\boldsymbol{P}=\mathbb{I}_{n}\otimes(GG^{\top})^{-\frac{1}{2}}\in\mathbb{R}^{d\times d}$ where $G = \nabla \mathcal{L}(\boldsymbol{W})$, we have the following inequalities:

    (i) With $\tilde{\beta} = \lambda_{\text{max}}(\boldsymbol{P}H)$, we have
    \begin{equation}
\mathcal{L}(\boldsymbol{W}')\leq\mathcal{L}(\boldsymbol{W})+\nabla\mathcal{L}(\boldsymbol{W})^{\top}(\boldsymbol{W}'-\boldsymbol{W})+\frac{\tilde{\beta}}{2}\Vert\boldsymbol{W}'-\boldsymbol{W}\Vert_{\boldsymbol{P}^{-1}}^{2},\label{eq:relative_smooth}
\end{equation}
where $\Vert\boldsymbol{U}\Vert_A = \boldsymbol{U}^\top A\boldsymbol{U}$.

(ii) With $\tilde{\alpha} = \lambda_{\text{min}}(\boldsymbol{P}H)$, we have
\begin{equation}
\frac{1}{2}\Vert\nabla\mathcal{L}(\boldsymbol{W})\Vert_{\boldsymbol{P}}^{2}\geq\tilde{\alpha}\left(\mathcal{L}(\boldsymbol{W})-\mathcal{L}^{*}\right).\label{eq:relative_PL}
\end{equation}
\end{lemma}

The proof of Lemma \ref{lem:smooth} can be found in Appendix \ref{sec:proof_smooth}, which lays foundation for Theorem \ref{thm:Muon_convergence_rate} about the convergence rate of Muon.

\begin{theorem}
    \label{thm:Muon_convergence_rate}
    Assuming the assumptions (A1) and (A2) and setting the learning rate $\eta = \frac{1}{\tilde{\beta}}$, we have
    \begin{equation}
\mathcal{L}\left(\boldsymbol{W}_{t}\right)-\mathcal{L}^{*}\leq\left(1-\frac{\tilde{\alpha}}{\tilde{\beta}}\right)^{t}\left(\mathcal{L}\left(\boldsymbol{W}_{0}\right)-\mathcal{L}^{*}\right).\label{eq:Muon_covergence}
\end{equation}
\end{theorem}

The proof of Theorem \ref{thm:Muon_convergence_rate} can be found in Appendix \ref{sec:Muon_convergence}. The following theorem quantifies the relationship between the convergence rates of SGD and Muon optimizers under the assumption that we use the K-FAC to approximate the Hessian matrix.

\begin{theorem}
\label{thm:comparision}
Assume that we use the K-FAC to estimate the Hessian matrix, i.e., $H \approx \boldsymbol{X}^{\top}\boldsymbol{X}\otimes G_{t}G_{t}^{\top}$. Then

\smallskip
\noindent
(i)\; $\displaystyle\frac{\alpha}{\beta}
    =\frac{\lambda_{\text{min}}(\boldsymbol{X}^{\top}\boldsymbol{X})\,\lambda_{\text{min}}(G_{t}G_{t}^{\top})}
           {\lambda_{\text{max}}(\boldsymbol{X}^{\top}\boldsymbol{X})\,\lambda_{\text{max}}(G_{t}G_{t}^{\top})},$

\smallskip
\noindent
(ii)\; $\displaystyle\frac{\tilde{\alpha}}{\tilde{\beta}}
    =\frac{\lambda_{\text{min}}(\boldsymbol{X}^{\top}\boldsymbol{X})\,\lambda_{\text{min}}(G_{t}G_{t}^{\top})^{1/2}}
           {\lambda_{\text{max}}(\boldsymbol{X}^{\top}\boldsymbol{X})\,\lambda_{\text{max}}(G_{t}G_{t}^{\top})^{1/2}},$

\smallskip
\noindent
(iii)\; $\displaystyle\frac{\alpha}{\beta}
    =\frac{\tilde{\alpha}}{\tilde{\beta}}\sqrt{\frac{\lambda_{\text{min}}(G_{t}G_{t}^{\top})}{\lambda_{\text{max}}(G_{t}G_{t}^{\top})}}
     <\frac{\tilde{\alpha}}{\tilde{\beta}}.$
\end{theorem}

Since $\frac{\tilde{\alpha}}{\tilde{\beta}}>\frac{\alpha}{\beta}$ from Theorem~\ref{thm:comparision}, we have $1-\frac{\tilde{\alpha}}{\tilde{\beta}}<1-\frac{\alpha}{\beta}$, showing that Muon converges faster than GD. The proof is deferred to Appendix~\ref{sec:proof_comparison}. Section~\ref{sec:exp_convergence} provides experimental validation of this convergence acceleration.

\section{Experiments}
\label{sec:experiments}

To empirically examine the theoretical insights developed in Section \ref{sec:theory}, we design two complementary sets of experiments comparing Muon against SGD. The first set focuses on stability and learning-rate sensitivity, while the second evaluates convergence speed under a more controlled training regime. Together, these experiments aim to probe two key aspects of our analysis: whether Muon can sustain substantially larger learning rates before divergence, consistent with its implicit spectral flattening, and whether it achieves faster convergence when both optimizers operate under the same learning rate.

\begin{figure*}[t]
    \centering
    
    \begin{subfigure}[t]{0.48\textwidth}
        \centering
        \includegraphics[width=\linewidth]{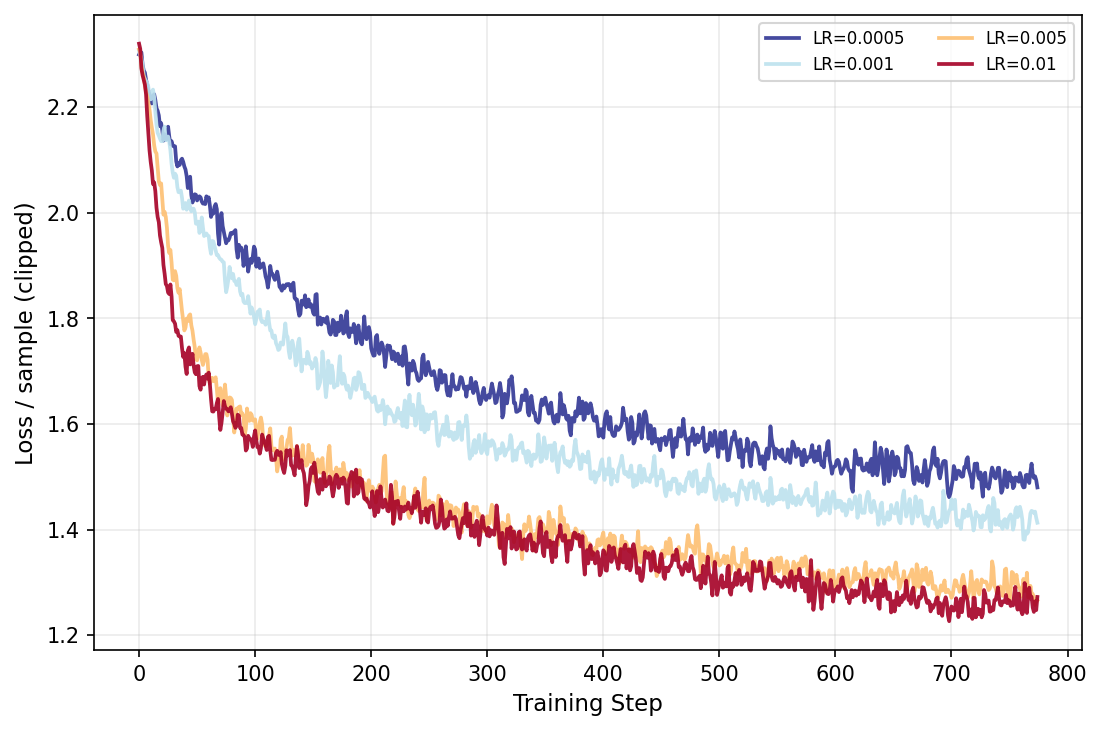}
        \caption{Muon}
    \end{subfigure}
    \hfill
    \begin{subfigure}[t]{0.48\textwidth}
        \centering
        \includegraphics[width=\linewidth]{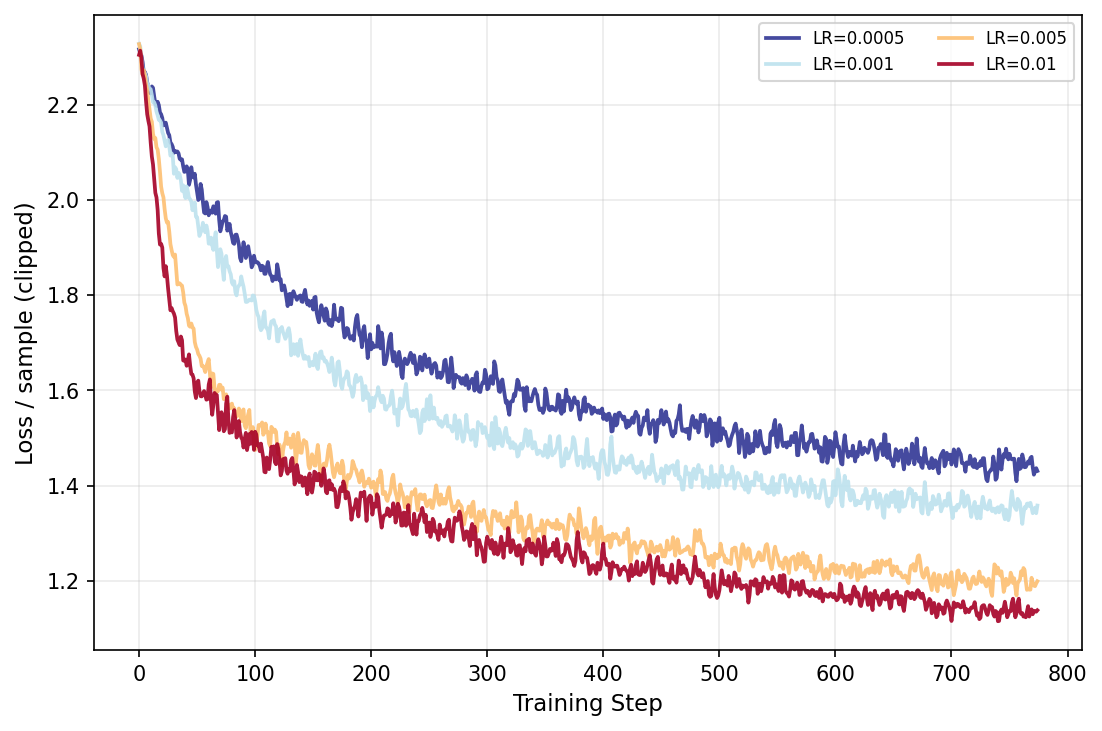}
        \caption{Muon + Momentum}
    \end{subfigure}

    \vspace{0.3cm}

    \begin{subfigure}[t]{0.48\textwidth}
        \centering
        \includegraphics[width=\linewidth]{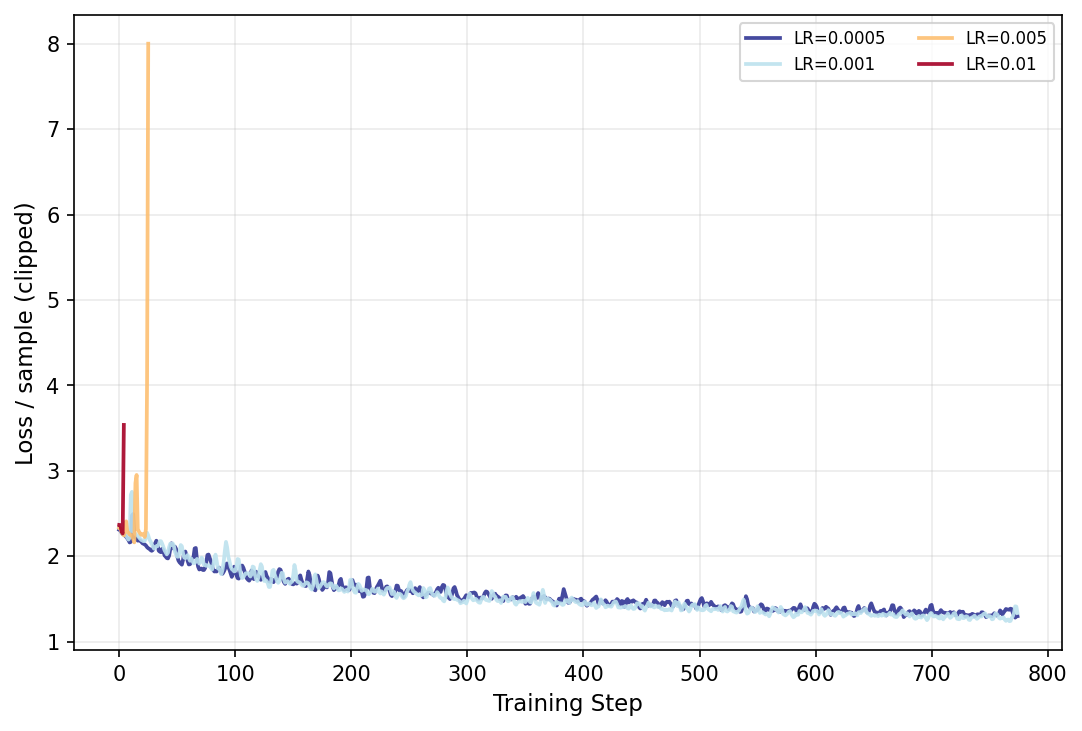}
        \caption{SGD}
    \end{subfigure}
    \hfill
    \begin{subfigure}[t]{0.48\textwidth}
        \centering
        \includegraphics[width=\linewidth]{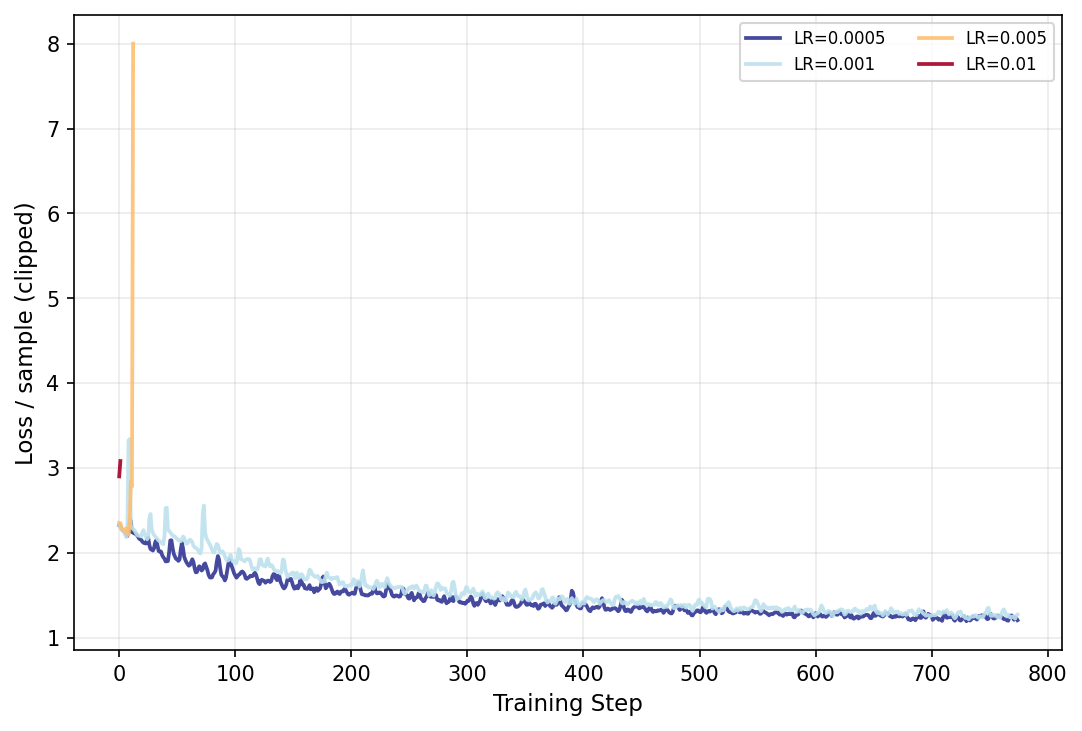}
        \caption{SGD + Momentum}
    \end{subfigure}

    \caption{
    Training loss curves under different learning rates. At higher learning rates, SGD diverges within the first few iterations, while Muon remains stable and continues to decrease the loss.
    }
    \label{fig:exp1_loss}
\end{figure*}

\textbf{General Experimental Setup and Isolation Strategy.}
All experiments use CIFAR-10 \citep{Krizhevsky2009LearningML} with CifarNet
\citep{7422783}. To isolate the effect of the optimizer under evaluation, we
adopt a dual-optimizer strategy: 4D convolutional kernels are updated by the
target optimizer (Muon or SGD) with learning rate $\eta_{\mathrm{conv}}$, while
all remaining parameters (biases, normalization parameters, classification head)
are updated by a fixed, carefully tuned reference optimizer. This ensures that
observed differences in stability and convergence reflect how each optimizer
handles the high-dimensional curvature of the convolutional layers, cleanly
separating their matrix-valued update behavior from the standard optimization
dynamics of the rest of the network.

\subsection{Maximal Stable Learning Rate}
\label{sec:exp_lr}


\textbf{Experimental Design.}
To directly observe the intrinsic stability of each update rule, we remove two components that can mask optimization issues. First, we eliminate all Batch Normalization (BN) layers \citep{ioffe2015batch}. BN re-centers and rescales activations, which can allow the classifier head and bias terms to continue learning even when convolutional weights become poorly conditioned. Without BN, instability propagates directly through the network. Second, we use a constant learning rate without scheduling, so each optimizer is evaluated under a fixed step size. We vary the convolutional learning rate over $\eta_{\mathrm{conv}} \in \{0.0005, 0.001, 0.005, 0.01\}$.



\textbf{Training Stability and Loss Trajectories.}
Figure~\ref{fig:exp1_loss} shows the training loss curves for both optimizers across the learning rate range. At smaller learning rates, both Muon and SGD reduce the training loss reliably. As the learning rate increases, their behaviors begin to diverge. SGD, with or without momentum, becomes unstable and its loss rapidly increases within the first few iterations. In contrast, Muon remains stable and continues to decrease the loss smoothly, even at learning rates where SGD fails. This behavior is consistent with the expected effect of spectral flattening, which moderates extreme update directions and improves robustness under larger step sizes.

We also conduct a mechanistic analysis of parameter and gradient norms during early training, which further supports the spectral flattening interpretation; details are provided in Appendix~\ref{app:mechanistic_norms}.

\subsubsection{A new perpestive on normalization layer.}
\label{sec:normalization}

\begin{figure}
    \centering
    \includegraphics[width=0.7\linewidth]{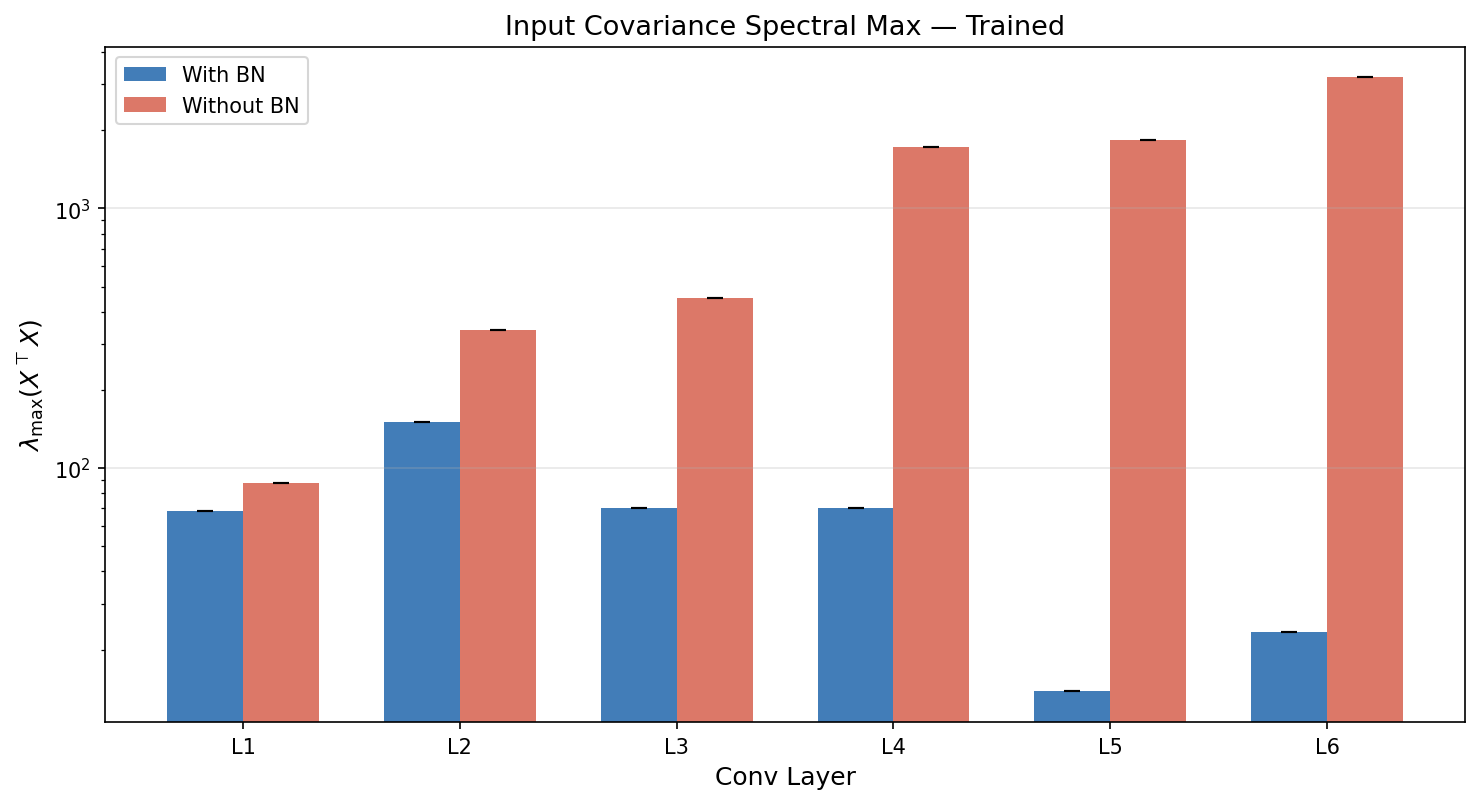}
    \caption{Layer-wise values of $\lambda_{\text{max}}(X^\top X)$ for trained CifarNet models with and without Batch Normalization. Batch Normalization substantially reduces the spectral scale of layer inputs, especially in deeper layers.}
    \label{fig:BN_max}
\end{figure}

\begin{figure}
    \centering
    \includegraphics[width=0.8\textwidth]{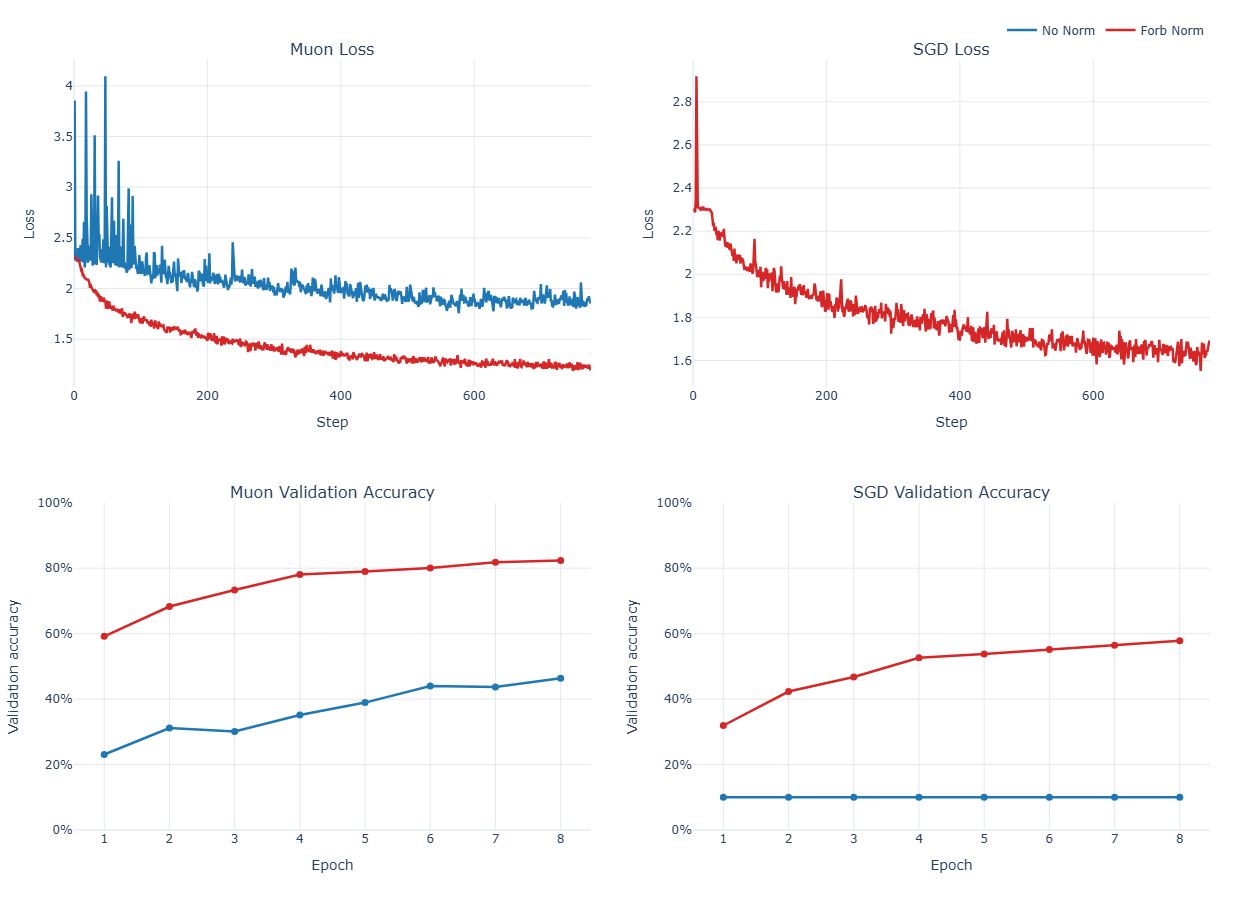}
    \caption{Training loss and validation accuracy of CifarNet with and without FrobNorm under Muon and SGD at 0.25 learning rate. FrobNorm improves stability and accuracy for both optimizers.}
    \label{fig:norm_cnn}
\end{figure}

\begin{figure*}[htpb] \centering \begin{subfigure}{0.48\textwidth} \includegraphics[width=\linewidth]{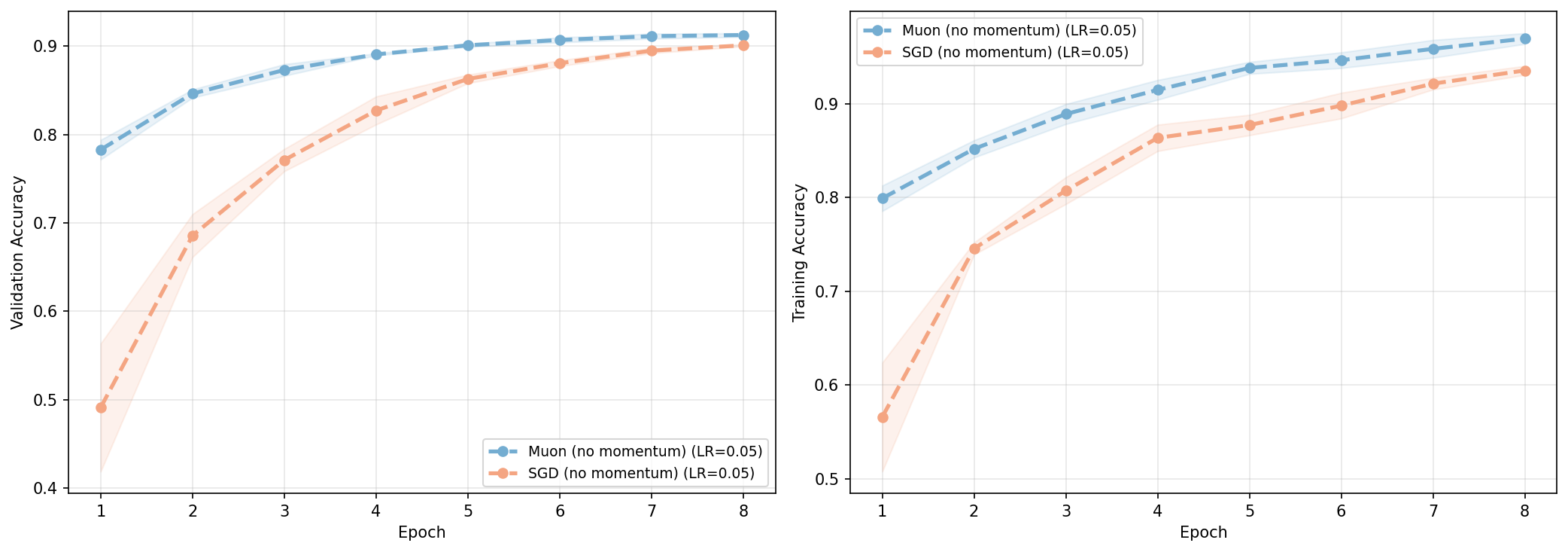} \caption{No Momentum} \end{subfigure}\hfill \begin{subfigure}{0.48\textwidth} \includegraphics[width=\linewidth]{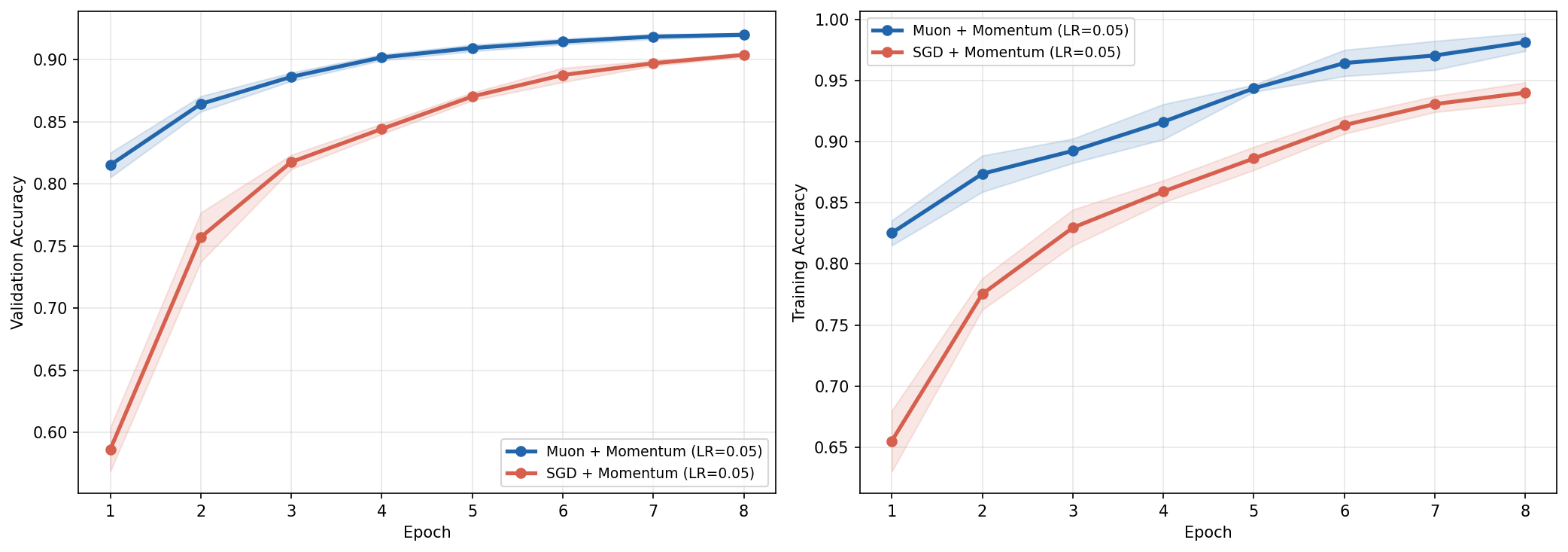} \caption{With Momentum} \end{subfigure} \caption{Training and Validation Accuracy vs. Epoch. Muon accelerates the learning process, achieving higher accuracy earlier than SGD despite identical learning rates.} \label{fig:exp2_acc} \end{figure*}

\textbf{Analysis.}
From Section~\ref{sec:lr_muon}, we observe that the maximal stable learning rate of a layer under both SGD and Muon is proportional to $\frac{1}{\lambda_{\text{max}}(X^\top X)}$, where $X$ is the input to that layer. This implies that if we can transform $X$ so that $\lambda_{\text{max}}(X^\top X)$ becomes smaller, the layer can tolerate a higher learning rate.

\textbf{Why Batch Normalization allows higher learning rates.}
It is widely observed that Batch Normalization enables higher learning rates, yet a formal
explanation has remained elusive. Our analysis suggests a mechanism: Batch Normalization
reduces $\lambda_{\text{max}}(X^\top X)$, which our theory identifies as the key quantity
governing the maximal stable learning rate. We verify this by computing
$\lambda_{\text{max}}(X^\top X)$ across layers of CifarNet trained with and without Batch
Normalization (Figure~\ref{fig:BN_max}). Without Batch Normalization,
$\lambda_{\text{max}}(X^\top X)$ grows rapidly with depth, reaching values on the order of
$10^3$ in later layers; with Batch Normalization, it remains small and stable. This indicates
that Batch Normalization does more than stabilize activation scales---it prevents the input
matrix $X$ from concentrating energy along a single dominant direction, directly enabling
larger learning rates.

\textbf{A new principle for designing normalization layers.}
Our analysis suggests a general design principle: to enable higher learning rates and faster convergence, a normalization layer should transform the input $X$ to reduce $\lambda_{\text{max}}(X^\top X)$. We test this with a simple scheme we call \emph{FrobNorm}. Since $\lambda_{\text{max}}(X^\top X) \le \|X\|_F^2$, normalizing as $\widetilde{X} = \frac{X}{\|X\|_F}$ guarantees $\lambda_{\text{max}}(\widetilde{X}^\top \widetilde{X}) \le 1$. We apply FrobNorm after each convolutional layer of CifarNet and train with SGD and Muon at the unusually high learning rate of 0.25. Figure~\ref{fig:norm_cnn} shows the results. FrobNorm consistently improves optimization over the no-normalization baseline: with Muon, it yields a smoother loss curve and higher validation accuracy ($>80\%$ vs.\ $<50\%$ after eight epochs); with SGD, the baseline fails to train entirely, while FrobNorm reaches nearly $60\%$ accuracy. Because FrobNorm directly enforces a spectral bound, these results support the principle that normalization should control the dominant spectral direction of the input rather than only its coordinate-wise scale. We examine this principle for Transformer architectures in Appendix~\ref{app:transformer}.

\subsection{Convergence Rate Acceleration}
\label{sec:exp_convergence}


\textbf{Experimental Design.} 
To accurately measure convergence speed in a standard deep learning environment, we restore both BN and a linear learning rate scheduler. Furthermore, we constrain both Muon and SGD to operate under the exact same learning rate ($\eta_{conv} = 0.05$). The rationale for this setup is twofold. First, BN and schedulers are ubiquitous in practice and necessary to achieve competitive absolute accuracy. Second, using identical learning rates strictly isolates the intrinsic efficiency of the optimizers. By leveling the playing field, we ensure that any observed acceleration is purely the result of Muon's spectral preconditioning, rather than the trivial byproduct of taking larger step sizes. Each configuration is executed over 5 independent runs to capture variance, with solid lines representing the mean and shaded regions indicating the standard deviation.

\begin{figure*}[htpb]
    \centering
    \begin{subfigure}{0.48\textwidth}
        \includegraphics[width=\linewidth]{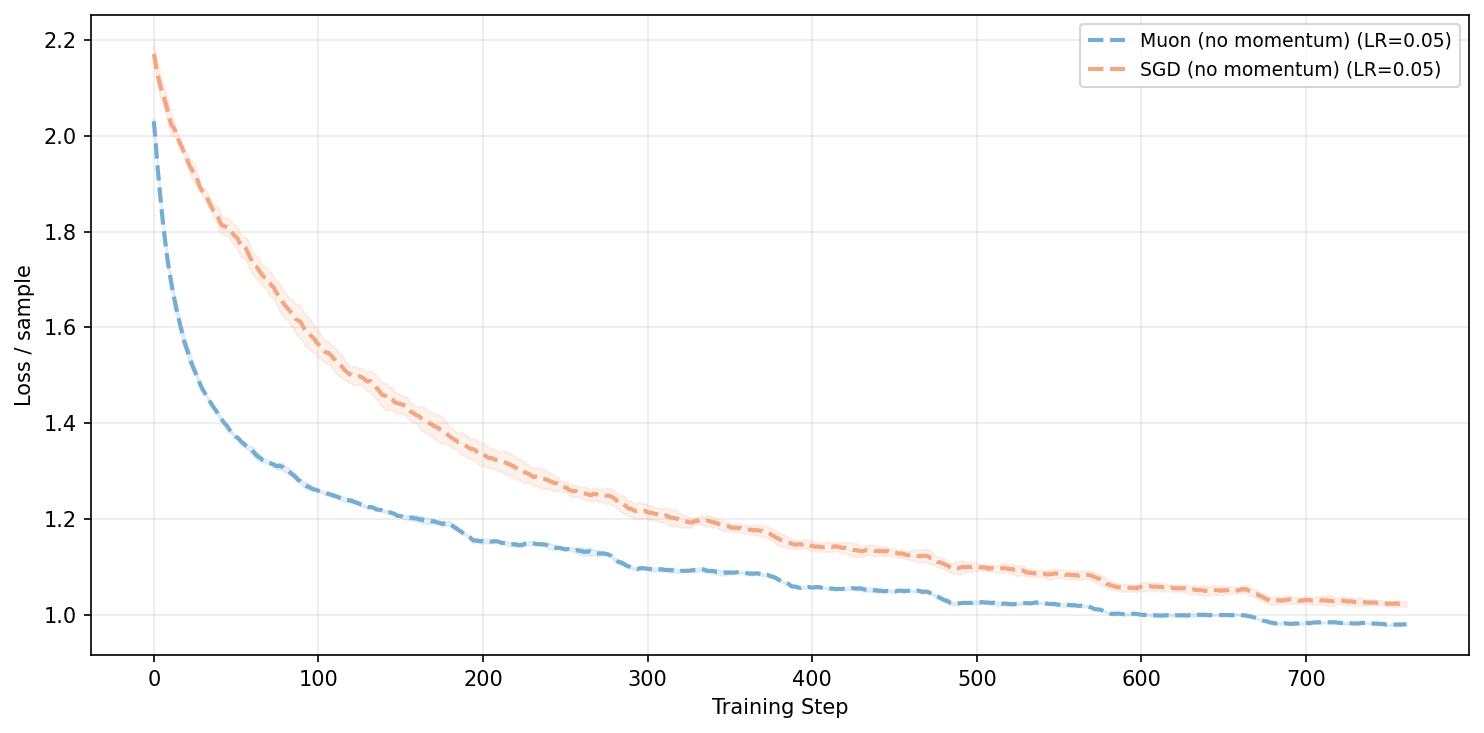}
        \caption{No Momentum}
    \end{subfigure}\hfill
    \begin{subfigure}{0.48\textwidth}
        \includegraphics[width=\linewidth]{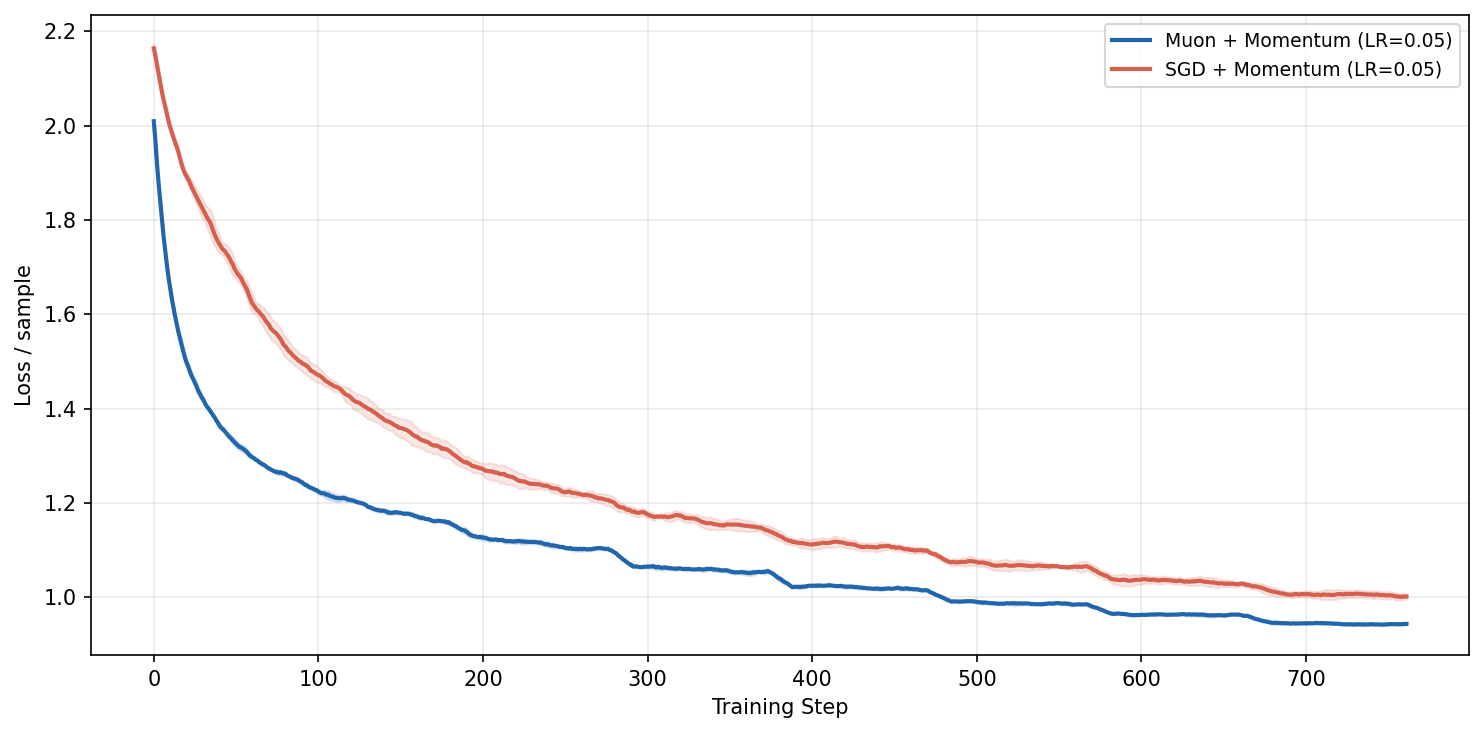}
        \caption{With Momentum}
    \end{subfigure}
    \caption{Training Loss vs. Step. Muon exhibits a consistently steeper loss descent compared to SGD.}
    \label{fig:exp2_loss}
\end{figure*}

\begin{figure*}[htpb]
    \centering
    \begin{subfigure}{0.48\textwidth}
        \includegraphics[width=\linewidth]{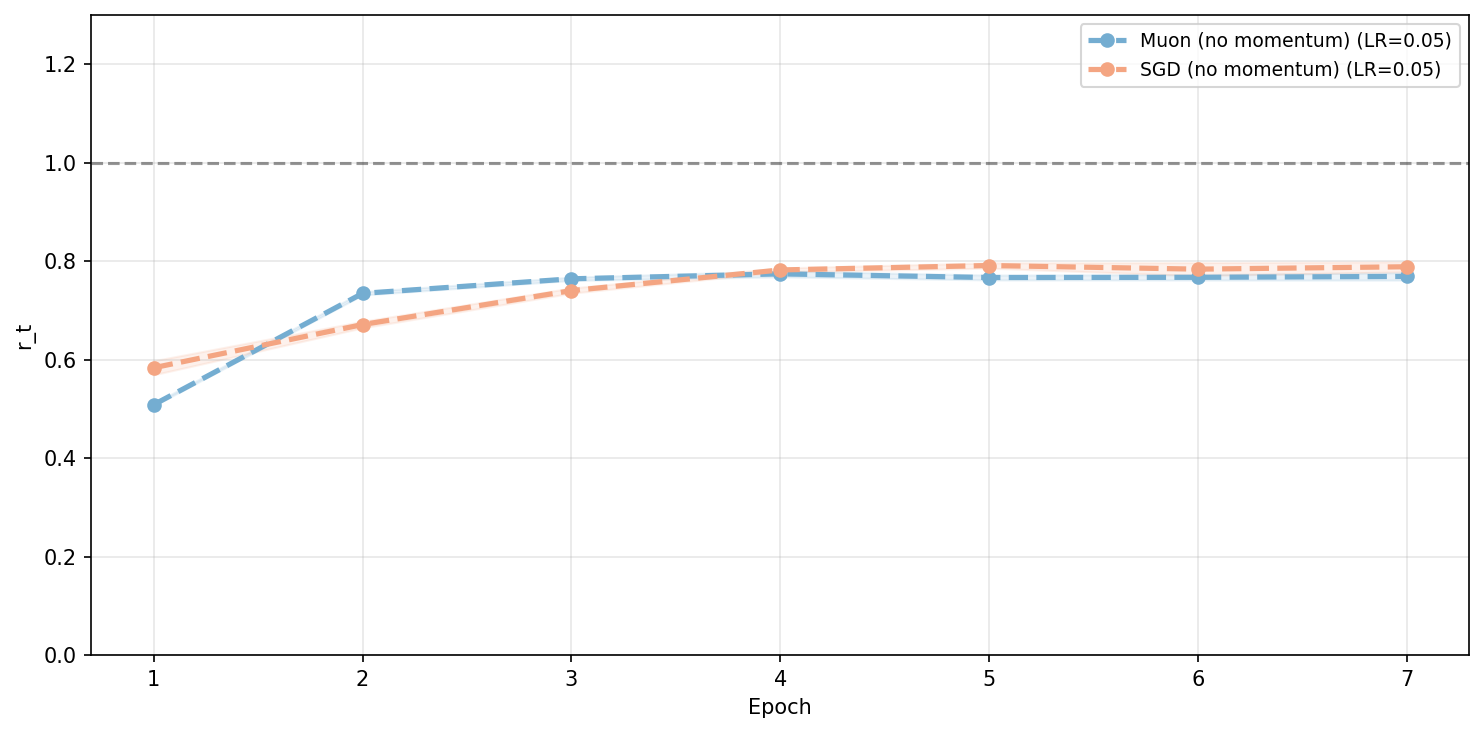}
        \caption{No Momentum}
    \end{subfigure}\hfill
    \begin{subfigure}{0.48\textwidth}
        \includegraphics[width=\linewidth]{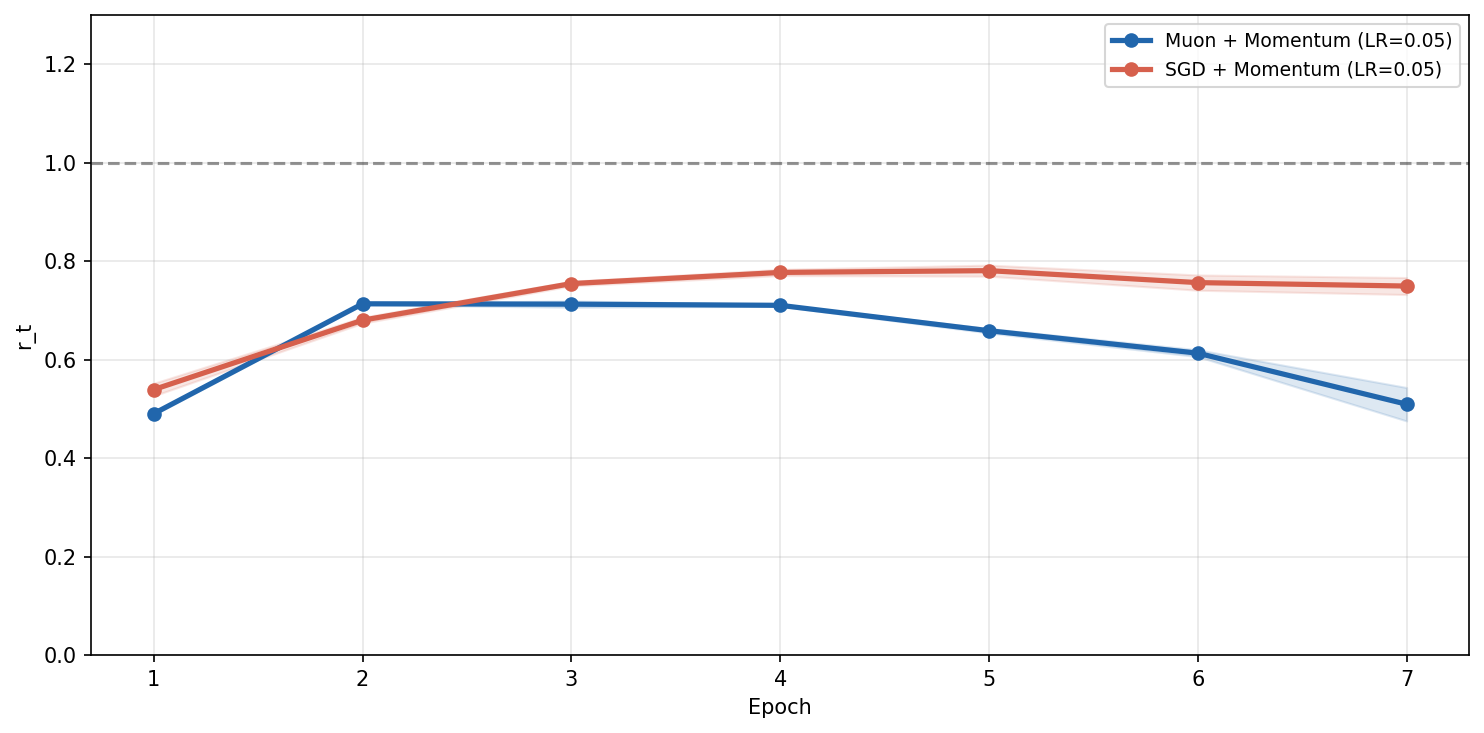}
        \caption{With Momentum}
    \end{subfigure}
    \caption{Empirical Convergence Ratio ($r_t$). A smaller $r_t$ implies a faster linear convergence rate. Muon maintains a consistently lower $r_t$ throughout training, supporting the preconditioned convergence improvement in Theorem \ref{thm:comparision}.}
    \label{fig:exp2_rt}
\end{figure*}

\textbf{Standard Convergence Metrics.}
Figure~\ref{fig:exp2_acc} and Figure~\ref{fig:exp2_loss} illustrate the training dynamics in terms of accuracy and loss. Even when operating at the exact same learning rate of $0.05$, Muon consistently outpaces SGD. The accuracy curves demonstrate that Muon climbs to higher validation and training accuracies significantly earlier in the training process. This is corroborated by the training loss trajectories, where Muon achieves a noticeably steeper and deeper descent. A detailed per-threshold breakdown is provided in Appendix~\ref{sec:threshold_analysis}, confirming that across all targeted validation milestones (from $70\%$ to $91\%$), Muon consistently arrives several epochs ahead of SGD.

\textbf{Theoretical Validation via Convergence Ratio.}
We directly validate the linear convergence rate via the convergence ratio:
\begin{equation}
r_t \;=\; \frac{\mathcal{L}(W_{t}) - \mathcal{L}^{*}}{\mathcal{L}(W_{t-1}) - \mathcal{L}^{*}}
\end{equation}
where $\mathcal{L}^{*}$ is estimated as the minimum loss across all runs. A smaller $r_t$ indicates faster convergence, reflecting a larger relative decrease in the optimality gap per iteration. As shown in Figure~\ref{fig:exp2_rt}, Muon's $r_t$ remains consistently lower than SGD's throughout training, confirming that Muon reduces a larger fraction of the remaining error per epoch, consistent with an improved effective convergence factor.

We additionally verify that these conclusions hold when each optimizer is given its own best learning rate ($\eta_{\mathrm{conv}} = 0.1$ for Muon, $0.01$ for SGD), reflecting a realistic tuning scenario. The results, presented in Appendix~\ref{app:convergence_best_lr}, confirm that Muon's convergence advantage persists under best-case tuning for both optimizers.

\section{Conclusion}

We have shown that Muon's orthogonalization step acts as \emph{spectral flattening}, raising the maximal stable learning rate from a bound set by the largest singular value to one set by an average singular scale, and that under a Kronecker-factored curvature model it improves the effective condition ratio $\tilde{\alpha}/\tilde{\beta}>\alpha/\beta$. Controlled experiments confirm both predictions. Our results help turn Muon's empirical success into a more precise theoretical picture, but they also leave several directions open. The convergence analysis is stated in the deterministic full-batch setting and uses the exact polar factor, while practical Muon uses stochastic gradients, finite Newton--Schulz iterations, momentum, schedules, and additional parameter groups. Extending the theory to these settings, and to larger-scale architectures where optimizer-system interactions matter, is an important next step.

\bibliographystyle{plainnat}
\bibliography{references}

@article{kingma2015adam,
  title={Adam: A method for stochastic optimization},
  author={Kingma, Diederik P and Ba, Jimmy},
  journal={arXiv preprint arXiv:1412.6980},
  year={2015}
}

@misc{jordan2024muon,
  title={Muon: An optimizer for hidden layers in neural networks},
  author={Jordan, Keller},
  year={2024},
  howpublished={\url{https://kellerjordan.github.io/posts/muon/}},
  note={Accessed: 2026-04-21}
}

@inproceedings{martens2015optimizing,
  title={Optimizing neural networks with kronecker-factored approximate curvature},
  author={Martens, James and Grosse, Roger},
  booktitle={International conference on machine learning},
  pages={2408--2417},
  year={2015},
  organization={PMLR}
}

@article{bernstein2024old,
  title={Old optimizer, new norm: An anthology},
  author={Bernstein, Jeremy and Newhouse, Laker},
  journal={arXiv preprint arXiv:2409.20325},
  year={2024}
}

@ARTICLE{7422783,
  author={Setio, Arnaud Arindra Adiyoso and Ciompi, Francesco and Litjens, Geert and Gerke, Paul and Jacobs, Colin and van Riel, Sarah J. and Wille, Mathilde Marie Winkler and Naqibullah, Matiullah and Sánchez, Clara I. and van Ginneken, Bram},
  journal={IEEE Transactions on Medical Imaging}, 
  title={Pulmonary Nodule Detection in CT Images: False Positive Reduction Using Multi-View Convolutional Networks}, 
  year={2016},
  volume={35},
  number={5},
  pages={1160-1169},
  doi={10.1109/TMI.2016.2536809}}

@inproceedings{Krizhevsky2009LearningML,
  title={Learning Multiple Layers of Features from Tiny Images},
  author={Alex Krizhevsky},
  year={2009},
  url={https://api.semanticscholar.org/CorpusID:18268744}
}

@inproceedings{ioffe2015batch,
  title={Batch normalization: Accelerating deep network training by reducing internal covariate shift},
  author={Ioffe, Sergey and Szegedy, Christian},
  booktitle={International conference on machine learning},
  pages={448--456},
  year={2015},
  organization={pmlr}
}

@article{liu2025muon,
  title={Muon is scalable for {LLM} training},
  author={Liu, Jingyuan and Su, Jianlin and Yao, Xingcheng and Jiang, Zhejun and Lai, Guokun and Du, Yulun and Qin, Yidao and others},
  journal={arXiv preprint arXiv:2502.16982},
  year={2025}
}

@article{shah2025practical,
  title={Practical efficiency of {Muon} for pretraining},
  author={Shah, Ishaan and Polloreno, Anthony M. and Stratos, Karl and Monk, Philip and Chaluvaraju, Adarsh and Hojel, Andrew and Ma, Andrew and Thomas, Anil and Tanwer, Ashish and Shah, Darsh J. and others},
  journal={arXiv preprint arXiv:2505.02222},
  year={2025}
}

@article{tveit2025grokking,
  title={Muon optimizer accelerates grokking},
  author={Tveit, Amund and Remseth, Bj{\o}rn and Skogvold, Arve},
  journal={arXiv preprint arXiv:2504.16041},
  year={2025}
}

@article{mehta2025muon,
  title={Muon: Training and trade-offs with latent attention and {MoE}},
  author={Mehta, Sushant and Dandekar, Raj and Dandekar, Rajat and Panat, Sreedath},
  journal={arXiv preprint arXiv:2509.24406},
  year={2025}
}

@article{page2025muonall,
  title={MuonAll: {Muon} variant for efficient finetuning of large language models},
  author={Page, Saurabh and Joshi, Advait and Sonawane, S. S.},
  journal={arXiv preprint arXiv:2511.06086},
  year={2025}
}

@article{gupta2025quantizing,
  title={On quantizing the state of the {Muon} optimizer},
  author={Gupta, Aman and Celente, Rafael and Shivanna, Abhishek and Braithwaite, Daniel Thomas and Dexter, Gregory and Tang, Shao and Udagawa, Hiroto and Silva, Daniel and Ramanath, Rohan and Keerthi, Sathiya},
  journal={arXiv preprint arXiv:2509.23106},
  year={2025}
}

@inproceedings{wen2026fantastic,
  title={Fantastic pretraining optimizers and where to find them},
  author={Wen, Kaiyue and Hall, David and Ma, Tengyu and Liang, Percy},
  booktitle={International Conference on Learning Representations},
  year={2026}
}

@inproceedings{boreiko2025towards,
  title={Towards understanding orthogonalization in {Muon}},
  author={Boreiko, Valentyn and Bu, Zhiqi and Zha, Sheng},
  booktitle={ICML Workshop on High-dimensional Learning Dynamics},
  year={2025}
}

@article{si2025adamuon,
  title={AdaMuon: Adaptive {Muon} optimizer},
  author={Si, Chongjie and Zhang, Debing and Shen, Wei},
  journal={arXiv preprint arXiv:2507.11005},
  year={2025}
}

@article{crawshaw2025exploration,
  title={An exploration of non-{E}uclidean gradient descent: {Muon} and its many variants},
  author={Crawshaw, Michael and Modi, Chirag and Liu, Mingrui and Gower, Robert M.},
  journal={arXiv preprint arXiv:2510.09827},
  year={2025}
}

@article{zhang2026adamimproves,
  title={Adam improves {Muon}: Adaptive moment estimation with orthogonalized momentum},
  author={Zhang, Minxin and Liu, Yuxuan and Schaeffer, Hayden},
  journal={arXiv preprint arXiv:2602.17080},
  year={2026}
}

@article{li2025note,
  title={A note on the convergence of {Muon} and further},
  author={Li, Jiaxiang and Hong, Mingyi},
  journal={arXiv preprint arXiv:2502.02900},
  year={2025}
}

@article{kovalev2025understanding,
  title={Understanding gradient orthogonalization for deep learning via non-{E}uclidean trust-region optimization},
  author={Kovalev, Dmitry},
  journal={arXiv preprint arXiv:2503.12645},
  year={2025}
}

@article{chen2025spectralconstraints,
  title={Muon optimizes under spectral norm constraints},
  author={Chen, Lizhang and Li, Jonathan and Liu, Qiang},
  journal={arXiv preprint arXiv:2506.15054},
  year={2025}
}

@article{davis2025spectralupdates,
  title={When do spectral gradient updates help in deep learning?},
  author={Davis, Damek and Drusvyatskiy, Dmitriy},
  journal={arXiv preprint arXiv:2512.04299},
  year={2025}
}

@inproceedings{wilson2017marginal,
  title={The marginal value of adaptive gradient methods in machine learning},
  author={Wilson, Ashia C and Roelofs, Rebecca and Stern, Mitchell and Srebro, Nati and Recht, Benjamin},
  booktitle={Advances in Neural Information Processing Systems},
  volume={30},
  year={2017}
}

@inproceedings{loshchilov2019decoupled,
  title={Decoupled weight decay regularization},
  author={Loshchilov, Ilya and Hutter, Frank},
  booktitle={International Conference on Learning Representations},
  year={2019}
}

@article{liu2022loss,
  title={Loss landscapes and optimization in over-parameterized non-linear systems and neural networks},
  author={Liu, Chaoyue and Zhu, Libin and Belkin, Mikhail},
  journal={Applied and Computational Harmonic Analysis},
  volume={59},
  pages={85--116},
  year={2022}
}


\appendix
\newpage
\section{Mathematical Proofs}\label{sec:proof}

\subsection{Full Derivations for Main-Text Equations}\label{sec:derivations}

\paragraph{Derivation of Equation~\eqref{eq:Taylor_expansion} (Quadratic Upper Bound).}\label{app:taylor}
Starting from a second-order Taylor expansion around $\boldsymbol{W}_{t-1}$ with perturbation $\Delta\boldsymbol{W}$:
\begin{align*}
\mathcal{L}\left(\boldsymbol{W}_{t-1}+\Delta\boldsymbol{W}\right)
&= \mathcal{L}\left(\boldsymbol{W}_{t-1}\right) + \nabla\mathcal{L}\left(\boldsymbol{W}_{t-1}\right)^{\top}\Delta\boldsymbol{W}
   + \frac{1}{2}\Delta\boldsymbol{W}^{\top}\nabla^{2}\mathcal{L}\left(\boldsymbol{W}_{t-1}\right)\Delta\boldsymbol{W}
   + o(\Vert\Delta\boldsymbol{W}\Vert_{2}^{2}) \\
&\approx \mathcal{L}\left(\boldsymbol{W}_{t-1}\right) + \nabla\mathcal{L}\left(\boldsymbol{W}_{t-1}\right)^{\top}\Delta\boldsymbol{W}
   + \frac{1}{2}\Delta\boldsymbol{W}^{\top}H\Delta\boldsymbol{W} \\
&\leq \mathcal{L}\left(\boldsymbol{W}_{t-1}\right) + \nabla\mathcal{L}\left(\boldsymbol{W}_{t-1}\right)^{\top}\Delta\boldsymbol{W}
   + \frac{\lambda_{\text{max}}^{H}}{2}\Vert\Delta\boldsymbol{W}\Vert_{2}^{2},
\end{align*}
where the first step is the exact second-order expansion, the second step approximates the Hessian by $H$, and the final inequality uses the Rayleigh bound $\Delta\boldsymbol{W}^{\top}H\Delta\boldsymbol{W} \leq \lambda_{\text{max}}^{H}\Vert\Delta\boldsymbol{W}\Vert_{2}^{2}$.

\paragraph{Derivation of Equation~\eqref{eq:gap_Muon} (Muon Loss Bound).}\label{app:muon_bound}
For the Muon optimizer, we set $\Delta\boldsymbol{W} = -\eta\boldsymbol{O}_{t}$ where $\boldsymbol{O}_{t} = U V^{\top}$ is the polar factor of the momentum buffer. Substituting into the quadratic bound~\eqref{eq:Taylor_expansion}:
\begin{align*}
\mathcal{L}\left(\boldsymbol{W}_{t-1}+\Delta\boldsymbol{W}\right)
&\leq \mathcal{L}\left(\boldsymbol{W}_{t-1}\right) - \eta\,\nabla\mathcal{L}\left(\boldsymbol{W}_{t-1}\right)^{\top}\boldsymbol{O}_{t}
   + \frac{\lambda_{\text{max}}^{H}}{2}\eta^{2}\Vert\boldsymbol{O}_{t}\Vert_{F}^{2} \\
&= \mathcal{L}\left(\boldsymbol{W}_{t-1}\right) - \eta\,\text{tr}\left(G_{t}^{\top}\boldsymbol{O}_{t}\right)
   + \frac{\lambda_{\text{max}}^{H}}{2}\eta^{2}\,\text{tr}\left(\boldsymbol{O}_{t}^{\top}\boldsymbol{O}_{t}\right) \\
&= \mathcal{L}\left(\boldsymbol{W}_{t-1}\right) - \eta\,\text{tr}\left(G_{t}^{\top}\boldsymbol{O}_{t}\right)
   + \frac{\lambda_{\text{max}}^{H}}{2}\eta^{2}\,\text{tr}\left(V^{\top}U^{\top}UV\right) \\
&= \mathcal{L}\left(\boldsymbol{W}_{t-1}\right) - \eta\,\text{tr}\left(G_{t}^{\top}\boldsymbol{O}_{t}\right)
   + \frac{\lambda_{\text{max}}^{H}}{2}\eta^{2}\,m,
\end{align*}
where the second line uses $\nabla\mathcal{L}(\boldsymbol{W}_{t-1})^{\top}\boldsymbol{O}_{t} = \text{tr}(G_{t}^{\top}\boldsymbol{O}_{t})$ and $\Vert\boldsymbol{O}_{t}\Vert_{F}^{2} = \text{tr}(\boldsymbol{O}_{t}^{\top}\boldsymbol{O}_{t})$, the third line substitutes the SVD $\boldsymbol{O}_{t} = UV^{\top}$, and the fourth uses $U^{\top}U = I$ (since $U$ has orthonormal columns) and $\text{tr}(V^{\top}V) = m$ (since $V$ is $m \times m$ orthogonal, assuming $m \leq n$).

\subsection{Proof of Theorem \ref{thm:gd_lr}}\label{sec:proof_gd_lr}
\begin{proof}
From~\eqref{eq:gd_eta_max}, the maximal learning rate for gradient descent is
\[
\eta_{t}^{\max} = \frac{2}{\lambda_{\text{max}}^{H}}\frac{\nabla\mathcal{L}\left(\boldsymbol{W}_{t-1}\right)^{\top}\boldsymbol{G}_{t}}{\Vert\boldsymbol{G}_{t}\Vert_{2}^{2}}.
\]
Since $\boldsymbol{G}_{t} = \nabla\mathcal{L}(\boldsymbol{W}_{t-1})$ by definition, the numerator simplifies to $\nabla\mathcal{L}(\boldsymbol{W}_{t-1})^{\top}\boldsymbol{G}_{t} = \Vert\boldsymbol{G}_{t}\Vert_{2}^{2}$, giving
\[
\eta_{t}^{\max} = \frac{2}{\lambda_{\text{max}}^{H}}\frac{\Vert\boldsymbol{G}_{t}\Vert_{2}^{2}}{\Vert\boldsymbol{G}_{t}\Vert_{2}^{2}} = \frac{2}{\lambda_{\text{max}}^{H}}.
\]
\end{proof}

\subsection{Proof of Theorem~\ref{thm:eta_Muon}}\label{sec:proof_eta_muon}
\begin{proof}
From the derivation in Appendix~\ref{app:muon_bound}, the loss after a Muon update $\Delta\boldsymbol{W} = -\eta\boldsymbol{O}_t$ satisfies
\[
\mathcal{L}(\boldsymbol{W}_{t}) \leq \mathcal{L}(\boldsymbol{W}_{t-1}) - \eta\,\mathrm{tr}(G_t^{\top}\boldsymbol{O}_t) + \frac{\lambda_{\text{max}}^{H}}{2}\eta^{2}m.
\]
For the loss to strictly decrease, we require the right-hand side to be smaller than $\mathcal{L}(\boldsymbol{W}_{t-1})$, which gives
\[
\eta < \frac{2}{\lambda_{\text{max}}^{H}}\frac{\mathrm{tr}(G_t^{\top}\boldsymbol{O}_t)}{m},
\]
and hence the maximal stable learning rate
\[
\eta_t^{\max} = \frac{2}{\lambda_{\text{max}}^{H}}\frac{\mathrm{tr}(G_t^{\top}\boldsymbol{O}_t)}{m}.
\]

Now, for the Muon optimizer without momentum ($\mu=0$), we have $\boldsymbol{O}_t = \text{Newton-Schulz}(G_t)$. Assuming exact Newton--Schulz iterations, let $G_t = U\Sigma V^{\top}$ be the SVD of the gradient matrix, where $U \in \mathbb{R}^{m \times m}$ and $V \in \mathbb{R}^{n \times m}$ have orthonormal columns (assuming $m \leq n$) and $\Sigma = \mathrm{diag}(\sigma_1, \ldots, \sigma_m)$ contains the singular values. The Newton--Schulz iteration converges to $\boldsymbol{O}_t = (G_t G_t^{\top})^{-1/2}G_t = UV^{\top}$.

Substituting the SVD:
\begin{align*}
\mathrm{tr}(G_t^{\top}\boldsymbol{O}_t) &= \mathrm{tr}\big((V\Sigma U^{\top})(UV^{\top})\big) \\
&= \mathrm{tr}(V\Sigma V^{\top}) \qquad (\text{since } U^{\top}U = I_m) \\
&= \mathrm{tr}(\Sigma V^{\top}V) \qquad (\text{by cyclic property of trace}) \\
&= \mathrm{tr}(\Sigma) \qquad (\text{since } V^{\top}V = I_m) \\
&= \sum_{i=1}^{m}\sigma_i.
\end{align*}

Plugging this into the expression for $\eta_t^{\max}$ yields
\[
\eta_t^{\max} = \frac{2}{\lambda_{\text{max}}^{H}}\frac{\sum_{i=1}^{m}\sigma_i}{m},
\]
which completes the proof.
\end{proof}

\subsection{Explanation of the Vectorial Update for Muon}\label{sec:vector_form}
We have the update in the matrix form
\begin{align*}
\boldsymbol{W}_{t} & =\boldsymbol{W}_{t-1}-\eta\left(G_{t}G_{t}^{\top}\right)^{-1/2}G_{t}\\
 & =\boldsymbol{W}_{t-1}-\eta\left(G_{t}G_{t}^{\top}\right)^{-1/2}G_{t}\mathbb{I}_{n}.
\end{align*}

Using the formula $\vecop(ABC)=\left(C^{\top}\otimes A\right)\vecop\left(B\right)$, we gain
\begin{align*}
\boldsymbol{W}_{t} & =\boldsymbol{W}_{t-1}-\eta\left(\mathbb{I}_{n}\otimes\left(G_{t}G_{t}^{\top}\right)^{-1/2}\right)\vecop\left(G_{t}\right)\\
 & =\boldsymbol{W}_{t-1}-\eta\boldsymbol{P}_{t}\boldsymbol{g}_{t}.
\end{align*}

\subsection{Proof of Lemma \ref{lem:smooth}}\label{sec:proof_smooth}
\begin{proof} We now prove both (i) and (ii).

(i) We start with
\begin{align*}
\mathcal{L}\left(\boldsymbol{W}'\right) & =\mathcal{L}\left(\boldsymbol{W}\right)+\left\langle \nabla\mathcal{L}\left(\boldsymbol{W}\right),\boldsymbol{W}'-\boldsymbol{W}\right\rangle +\frac{1}{2}\left(\boldsymbol{W}'-\boldsymbol{W}\right)^{\top}H\left(\boldsymbol{W}'-\boldsymbol{W}\right)+o\left(\Vert\boldsymbol{W}'-\boldsymbol{W}\Vert_{2}^{2}\right)\\
 & \approx\mathcal{L}\left(\boldsymbol{W}\right)+\left\langle \nabla\mathcal{L}\left(\boldsymbol{W}\right),\boldsymbol{W}'-\boldsymbol{W}\right\rangle +\frac{1}{2}\left(\boldsymbol{W}'-\boldsymbol{W}\right)^{\top}H\left(\boldsymbol{W}'-\boldsymbol{W}\right).
\end{align*}
We now prove that $\left(\boldsymbol{W}'-\boldsymbol{W}\right)^{\top}H\left(\boldsymbol{W}'-\boldsymbol{W}\right)\le \tilde{\beta}\Vert\boldsymbol{W}'-\boldsymbol{W}\Vert_{\boldsymbol{P}^{-1}}^{2}$. Let $\boldsymbol{v}= \boldsymbol{W}'-\boldsymbol{W}$ and $\boldsymbol{u} = \boldsymbol{P}^{-1/2}\boldsymbol{v}$ or $\boldsymbol{v} = \boldsymbol{P}^{1/2}\boldsymbol{u}$. We then have
\[
\boldsymbol{v}^{\top}H\boldsymbol{v}=\left(\boldsymbol{P}^{1/2}\boldsymbol{u}\right)^{\top}H\boldsymbol{P}^{1/2}\boldsymbol{u}=\boldsymbol{u}^{\top}\left(\boldsymbol{P}^{1/2}H\boldsymbol{P}^{1/2}\right)\boldsymbol{u}=\boldsymbol{u}^{\top}\boldsymbol{Q}\boldsymbol{u},
\]
where $\boldsymbol{Q}=\boldsymbol{P}^{1/2}H\boldsymbol{P}^{1/2}$.

Moreover, we have
\begin{align*}
\boldsymbol{u}^{\top}\boldsymbol{Q}\boldsymbol{u} & \leq\lambda_{max}\left(\boldsymbol{P}^{1/2}H\boldsymbol{P}^{1/2}\right)\Vert\boldsymbol{u}\Vert_{2}^{2},\\
\left(\boldsymbol{P}^{-1/2}\boldsymbol{v}\right)^{\top} & \left(\boldsymbol{P}^{1/2}H\boldsymbol{P}^{1/2}\right)\boldsymbol{P}^{-1/2}\boldsymbol{v}\leq\lambda_{max}\left(\boldsymbol{P}^{1/2}H\boldsymbol{P}^{1/2}\right)\Vert\boldsymbol{u}\Vert_{2}^{2},\\
\boldsymbol{v}^{\top}H\boldsymbol{v} & \leq\lambda_{max}\left(\boldsymbol{P}^{1/2}H\boldsymbol{P}^{1/2}\right)\Vert\boldsymbol{u}\Vert_{2}^{2}.
\end{align*}
It appears that
\[
\Vert\boldsymbol{u}\Vert_{2}^{2}=\boldsymbol{u}^{\top}\boldsymbol{u}=\left(\boldsymbol{P}^{-1/2}\boldsymbol{v}\right)^{\top}\boldsymbol{P}^{-1/2}\boldsymbol{v}=\boldsymbol{v}^{\top}\boldsymbol{P}^{-1}\boldsymbol{v}=\Vert\boldsymbol{v}\Vert_{\boldsymbol{P}^{-1}}.
\]
Therefore, we have
\[
\boldsymbol{v}^{\top}H\boldsymbol{v}\leq\lambda_{max}\left(\boldsymbol{P}^{1/2}H\boldsymbol{P}^{1/2}\right)\Vert\boldsymbol{v}\Vert_{\boldsymbol{P}^{-1}}^{2}=\tilde{\beta}\Vert\boldsymbol{v}\Vert_{\boldsymbol{P}^{-1}}^{2}.
\]

This concludes our proof with the note that because $\boldsymbol{P^{1/2}H\boldsymbol{P}^{1/2}}$ and $H\boldsymbol{P}$ are similar, their eigen-values are the same.

(ii) We start with
\begin{align*}
\mathcal{L}\left(\boldsymbol{W}'\right) & =\mathcal{L}\left(\boldsymbol{W}\right)+\left\langle \nabla\mathcal{L}\left(\boldsymbol{W}\right),\boldsymbol{W}'-\boldsymbol{W}\right\rangle +\frac{1}{2}\left(\boldsymbol{W}'-\boldsymbol{W}\right)^{\top}H\left(\boldsymbol{W}'-\boldsymbol{W}\right)+o\left(\Vert\boldsymbol{W}'-\boldsymbol{W}\Vert_{2}^{2}\right)\\
 & \approx\mathcal{L}\left(\boldsymbol{W}\right)+\left\langle \nabla\mathcal{L}\left(\boldsymbol{W}\right),\boldsymbol{W}'-\boldsymbol{W}\right\rangle +\frac{1}{2}\left(\boldsymbol{W}'-\boldsymbol{W}\right)^{\top}H\left(\boldsymbol{W}'-\boldsymbol{W}\right).
\end{align*}
We now prove that $\left(\boldsymbol{W}'-\boldsymbol{W}\right)^{\top}H\left(\boldsymbol{W}'-\boldsymbol{W}\right)\ge \tilde{\alpha}\Vert\boldsymbol{W}'-\boldsymbol{W}\Vert_{\boldsymbol{P}^{-1}}^{2}$. Let $\boldsymbol{v}= \boldsymbol{W}'-\boldsymbol{W}$ and $\boldsymbol{u} = \boldsymbol{P}^{-1/2}\boldsymbol{v}$ or $\boldsymbol{v} = \boldsymbol{P}^{1/2}\boldsymbol{u}$. We then have
\[
\boldsymbol{v}^{\top}H\boldsymbol{v}=\left(\boldsymbol{P}^{1/2}\boldsymbol{u}\right)^{\top}H\boldsymbol{P}^{1/2}\boldsymbol{u}=\boldsymbol{u}^{\top}\left(\boldsymbol{P}^{1/2}H\boldsymbol{P}^{1/2}\right)\boldsymbol{u}=\boldsymbol{u}^{\top}\boldsymbol{Q}\boldsymbol{u},
\]
where $\boldsymbol{Q}=\boldsymbol{P}^{1/2}H\boldsymbol{P}^{1/2}$.

Moreover, we have
\begin{align*}
\boldsymbol{u}^{\top}\boldsymbol{Q}\boldsymbol{u} & \geq\lambda_{min}\left(\boldsymbol{P}^{1/2}H\boldsymbol{P}^{1/2}\right)\Vert\boldsymbol{u}\Vert_{2}^{2},\\
\left(\boldsymbol{P}^{-1/2}\boldsymbol{v}\right)^{\top} & \left(\boldsymbol{P}^{1/2}H\boldsymbol{P}^{1/2}\right)\boldsymbol{P}^{-1/2}\boldsymbol{v}\geq\lambda_{min}\left(\boldsymbol{P}^{1/2}H\boldsymbol{P}^{1/2}\right)\Vert\boldsymbol{u}\Vert_{2}^{2},\\
\boldsymbol{v}^{\top}H\boldsymbol{v} & \geq\lambda_{min}\left(\boldsymbol{P}^{1/2}H\boldsymbol{P}^{1/2}\right)\Vert\boldsymbol{u}\Vert_{2}^{2} = \lambda_{min}\left(H\boldsymbol{P}\right)\Vert\boldsymbol{u}\Vert_{2}^{2}.
\end{align*}
Here we note that because $\boldsymbol{P^{1/2}H\boldsymbol{P}^{1/2}}$ and $H\boldsymbol{P}$ are similar, their eigen-values are the same.

It appears that
\[
\Vert\boldsymbol{u}\Vert_{2}^{2}=\boldsymbol{u}^{\top}\boldsymbol{u}=\left(\boldsymbol{P}^{-1/2}\boldsymbol{v}\right)^{\top}\boldsymbol{P}^{-1/2}\boldsymbol{v}=\boldsymbol{v}^{\top}\boldsymbol{P}^{-1}\boldsymbol{v}=\Vert\boldsymbol{v}\Vert_{\boldsymbol{P}^{-1}}.
\]
Therefore, we have
\[
\boldsymbol{v}^{\top}H\boldsymbol{v}\geq\lambda_{min}\left(\boldsymbol{P}^{1/2}H\boldsymbol{P}^{1/2}\right)\Vert\boldsymbol{v}\Vert_{\boldsymbol{P}^{-1}}^{2}=\tilde{\alpha}\Vert\boldsymbol{v}\Vert_{\boldsymbol{P}^{-1}}^{2}.
\] 
\[
\mathcal{L}\left(\boldsymbol{W}'\right)\geq\mathcal{L}\left(\boldsymbol{W}\right)+\left\langle \nabla\mathcal{L}\left(\boldsymbol{W}\right),\boldsymbol{W}'-\boldsymbol{W}\right\rangle +\frac{1}{2}\tilde{\alpha}\Vert\boldsymbol{W}'-\boldsymbol{W}\Vert_{\boldsymbol{P}^{-1}}^{2}.
\]
By choosing $\boldsymbol{W}' = \boldsymbol{W}^*$, we obtain
\[
\mathcal{L}\left(\boldsymbol{W}^{*}\right)\geq\mathcal{L}\left(\boldsymbol{W}\right)+\left\langle \nabla\mathcal{L}\left(\boldsymbol{W}\right),\boldsymbol{W}^{*}-\boldsymbol{W}\right\rangle +\frac{1}{2}\tilde{\alpha}\Vert\boldsymbol{W}^{*}-\boldsymbol{W}\Vert_{\boldsymbol{P}^{-1}}^{2}.
\]
Rearranging the terms, we reach
\begin{align*}
\mathcal{L}\left(\boldsymbol{W}\right)-\mathcal{L}^{*} & \leq\left\langle \nabla\mathcal{L}\left(\boldsymbol{W}\right),\boldsymbol{W}^{*}-\boldsymbol{W}\right\rangle -\frac{1}{2}\tilde{\alpha}\Vert\boldsymbol{W}^{*}-\boldsymbol{W}\Vert_{\boldsymbol{P}^{-1}}^{2}\\
 & \leq\Vert\nabla\mathcal{L}\left(\boldsymbol{W}\right)\Vert_{\boldsymbol{P}}\Vert\boldsymbol{W}^{*}-\boldsymbol{W}\Vert_{\boldsymbol{P}^{-1}}-\frac{1}{2}\tilde{\alpha}\Vert\boldsymbol{W}^{*}-\boldsymbol{W}\Vert_{\boldsymbol{P}^{-1}}^{2}.
\end{align*}
Here we note that $\Vert \cdot \Vert_{\boldsymbol{P}^{-1}}$ is a dual-norm of $\Vert \cdot \Vert_{\boldsymbol{P}}$. 

Denote $r= \Vert\boldsymbol{W}^{*}-\boldsymbol{W}\Vert_{\boldsymbol{P}^{-1}}^{2}$. We consider the function $f(r)=\Vert\nabla\mathcal{L}\left(\boldsymbol{W}\right)\Vert_{\boldsymbol{P}}r-\tilde{\alpha}r^{2}$, which has the global maximum at
\[
r_{max}=\frac{\Vert\nabla\mathcal{L}\left(\boldsymbol{W}\right)\Vert_{\boldsymbol{P}}}{\tilde{\alpha}}\rightarrow f_{max}=\frac{1}{2}\frac{\Vert\nabla\mathcal{L}\left(\boldsymbol{W}\right)\Vert_{\boldsymbol{P}}^{2}}{\tilde{\alpha}}.
\]

This leads to
\[
\mathcal{L}\left(\boldsymbol{W}\right)-\mathcal{L}^{*}\leq\frac{1}{2}\frac{\Vert\nabla\mathcal{L}\left(\boldsymbol{W}\right)\Vert_{\boldsymbol{P}}^{2}}{\tilde{\alpha}}.
\]

\end{proof}

\subsection{Proof of Theorem \ref{thm:GD_converenge_rate}}\label{sec:proof_gd_conv}
\begin{proof}
    We start with
    \[
\mathcal{L}\left(\boldsymbol{W}_{t}\right)\leq\mathcal{L}\left(\boldsymbol{W}_{t-1}\right)+\left\langle \nabla\mathcal{L}\left(\boldsymbol{W}_{t-1}\right),\boldsymbol{W}_{t}-\boldsymbol{W}_{t-1}\right\rangle +\frac{1}{2}\beta\Vert\boldsymbol{W}_{t}-\boldsymbol{W}_{t-1}\Vert_{2}^{2}.
\]
This follows that
\begin{align*}
\mathcal{L}\left(\boldsymbol{W}_{t}\right) & \leq\mathcal{L}\left(\boldsymbol{W}_{t-1}\right)+\left\langle \nabla\mathcal{L}\left(\boldsymbol{W}_{t-1}\right),-\eta\boldsymbol{g}_{t}\right\rangle +\frac{1}{2}\eta^{2}\beta\Vert\boldsymbol{g}_{t}\Vert_{2}^{2}\\
 & =\mathcal{L}\left(\boldsymbol{W}_{t-1}\right)+\left\langle \boldsymbol{g}_{t},-\eta\boldsymbol{g}_{t}\right\rangle +\frac{1}{2}\eta^{2}\beta\Vert\boldsymbol{g}_{t}\Vert_{2}^{2}\\
 & =\mathcal{L}\left(\boldsymbol{W}_{t-1}\right)-\eta\Vert\boldsymbol{g}_{t}\Vert_{2}^{2}+\frac{1}{2}\eta^{2}\beta\Vert\boldsymbol{g}_{t}\Vert_{2}^{2}\\
 & =\mathcal{L}\left(\boldsymbol{W}_{t-1}\right)-\eta\Vert\boldsymbol{g}_{t}\Vert_{2}^{2}\left(1-\frac{1}{2}\eta\beta\right).
\end{align*}

Choosing $\eta = \frac{1}{\beta}$, we have
\[
\mathcal{L}\left(\boldsymbol{W}_{t}\right)\leq\mathcal{L}\left(\boldsymbol{W}_{t-1}\right)-\frac{1}{2\beta}\Vert\boldsymbol{g}_{t}\Vert_{2}^{2}.
\]

Linking to PL-condition, we have
\begin{align*}
\mathcal{L}\left(\boldsymbol{W}_{t}\right) & \leq\mathcal{L}\left(\boldsymbol{W}_{t-1}\right)-\frac{\alpha}{\beta}\left(\mathcal{L}\left(\boldsymbol{W}_{t-1}\right)-\mathcal{L}^{*}\right)\\
\mathcal{L}\left(\boldsymbol{W}_{t}\right)-\mathcal{L}^{*} & \leq\left(1-\frac{\alpha}{\beta}\right)\left(\mathcal{L}\left(\boldsymbol{W}_{t-1}\right)-\mathcal{L}^{*}\right).
\end{align*}
\end{proof}

\subsection{Proof of Theorem \ref{thm:Muon_convergence_rate}} \label{sec:Muon_convergence}
\begin{proof}
We start with
\[
\mathcal{L}\left(\boldsymbol{W}_{t}\right)\leq\mathcal{L}\left(\boldsymbol{W}_{t-1}\right)+\left\langle \nabla\mathcal{L}\left(\boldsymbol{W}_{t-1}\right),\boldsymbol{W}_{t}-\boldsymbol{W}_{t-1}\right\rangle +\frac{1}{2}\tilde{\beta}\Vert\boldsymbol{W}_{t}-\boldsymbol{W}_{t-1}\Vert_{\boldsymbol{P}_{t}^{-1}}^{2}.
\]It appears that
\begin{align*}
\Vert\boldsymbol{W}_{t}-\boldsymbol{W}_{t-1}\Vert_{\boldsymbol{P}_{t}^{-1}}^{2} & =\eta^{2}\left(\boldsymbol{P}_{t}\boldsymbol{g}_{t}\right)^{\top}\boldsymbol{P}_{t}^{-1}\boldsymbol{P}_{t}\boldsymbol{g}_{t}\\
= & \eta^{2}\boldsymbol{g}_{t}^{\top}\boldsymbol{P}_{t}\boldsymbol{g}_{t}=\eta^{2}\Vert\boldsymbol{g}_{t}\Vert_{\boldsymbol{P}_{t}}^{2}.
\end{align*}
This follows that
\begin{align*}
\mathcal{L}\left(\boldsymbol{W}_{t}\right) & \leq\mathcal{L}\left(\boldsymbol{W}_{t-1}\right)+\left\langle \nabla\mathcal{L}\left(\boldsymbol{W}_{t-1}\right),-\eta\boldsymbol{P}_{t}\boldsymbol{g}_{t}\right\rangle +\frac{1}{2}\eta^{2}\tilde{\beta}\Vert\boldsymbol{g}_{t}\Vert_{\boldsymbol{P}_{t}}^{2}\\
 & =\mathcal{L}\left(\boldsymbol{W}_{t-1}\right)+\left\langle \boldsymbol{g}_{t},-\eta\boldsymbol{P}_{t}\boldsymbol{g}_{t}\right\rangle +\frac{1}{2}\eta^{2}\tilde{\beta}\Vert\boldsymbol{g}_{t}\Vert_{\boldsymbol{P}_{t}}^{2}\\
 & =\mathcal{L}\left(\boldsymbol{W}_{t-1}\right)-\eta\Vert\boldsymbol{g}_{t}\Vert_{\boldsymbol{P}_{t}}^{2}+\frac{1}{2}\eta^{2}\tilde{\beta}\Vert\boldsymbol{g}_{t}\Vert_{\boldsymbol{P}_{t}}^{2}\\
 & =\mathcal{L}\left(\boldsymbol{W}_{t-1}\right)-\eta\Vert\boldsymbol{g}_{t}\Vert_{\boldsymbol{P}_{t}}^{2}\left(1-\frac{1}{2}\eta\tilde{\beta}\right).
\end{align*}
By choosing $\eta = \frac{1}{\tilde{\beta}}$, this becomes
\[
\mathcal{L}\left(\boldsymbol{W}_{t}\right)\leq\mathcal{L}\left(\boldsymbol{W}_{t-1}\right)-\frac{1}{2\tilde{\beta}}\Vert\boldsymbol{g}_{t}\Vert_{\boldsymbol{P}_{t}}^{2}.
\]

Linking to the PL-condition, we obtain
\begin{align*}
\mathcal{L}\left(\boldsymbol{W}_{t}\right) & \leq\mathcal{L}\left(\boldsymbol{W}_{t-1}\right)-\frac{\tilde{\alpha}}{\tilde{\beta}}\left(\mathcal{L}\left(\boldsymbol{W}_{t-1}\right)-\mathcal{L}^{*}\right)\\
\mathcal{L}\left(\boldsymbol{W}_{t}\right)-\mathcal{L}^{*} & \leq\left(1-\frac{\tilde{\alpha}}{\tilde{\beta}}\right)\left(\mathcal{L}\left(\boldsymbol{W}_{t-1}\right)-\mathcal{L}^{*}\right).
\end{align*}

\end{proof}

\subsection{Proof of Theorem \ref{thm:comparision}}\label{sec:proof_comparison}
\begin{proof}
We derive as
\begin{align*}
\beta & =\lambda_{\text{max}}\left(H\right)\approx\lambda_{\text{max}}\left(\boldsymbol{X}^{\top}\boldsymbol{X}\otimes G_{t}G_{t}^{\top}\right)=\lambda_{\text{max}}\left(\boldsymbol{X}^{\top}\boldsymbol{X}\right)\times\lambda_{\text{max}}\left(G_{t}G_{t}^{\top}\right),\\
\alpha & =\lambda_{\text{min}}\left(H\right)\approx\lambda_{\text{min}}\left(\boldsymbol{X}^{\top}\boldsymbol{X}\otimes G_{t}G_{t}^{\top}\right)=\lambda_{\text{min}}\left(\boldsymbol{X}^{\top}\boldsymbol{X}\right)\times\lambda_{\text{min}}\left(G_{t}G_{t}^{\top}\right).
\end{align*}
\begin{align*}
\tilde{\beta} & =\lambda_{\text{max}}\left(\boldsymbol{P}_{t}H\right)=\lambda_{\text{max}}\left(\left(\mathbb{I}_{n}\otimes\left(G_{t}G_{t}^{\top}\right)^{-1/2}\right)\left(\boldsymbol{X}^{\top}\boldsymbol{X}\otimes G_{t}G_{t}^{\top}\right)\right)\\
 & =\lambda_{\text{max}}\left(\left(\mathbb{I}_{n}\boldsymbol{X}^{\top}\boldsymbol{X}\right)\otimes\left(\left(G_{t}G_{t}^{\top}\right)^{-1/2}G_{t}G_{t}^{\top}\right)\right)\\
 & =\lambda_{\text{max}}\left(\boldsymbol{X}^{\top}\boldsymbol{X}\otimes\left(G_{t}G_{t}^{\top}\right)^{1/2}\right)=\lambda_{\text{max}}\left(\boldsymbol{X}^{\top}\boldsymbol{X}\right)\lambda_{\text{max}}\left(\left(G_{t}G_{t}^{\top}\right)^{1/2}\right)\\
 & =\lambda_{\text{max}}\left(\boldsymbol{X}^{\top}\boldsymbol{X}\right)\lambda_{\text{max}}\left(G_{t}G_{t}^{\top}\right)^{1/2}.
\end{align*}
 \begin{align*}
\tilde{\alpha} & =\lambda_{\text{min}}\left(\boldsymbol{P}_{t}H\right)=\lambda_{\text{min}}\left(\left(\mathbb{I}_{n}\otimes\left(G_{t}G_{t}^{\top}\right)^{-1/2}\right)\left(\boldsymbol{X}^{\top}\boldsymbol{X}\otimes G_{t}G_{t}^{\top}\right)\right)\\
 & =\lambda_{\text{min}}\left(\left(\mathbb{I}_{n}\boldsymbol{X}^{\top}\boldsymbol{X}\right)\otimes\left(\left(G_{t}G_{t}^{\top}\right)^{-1/2}G_{t}G_{t}^{\top}\right)\right)\\
 & =\lambda_{\text{min}}\left(\boldsymbol{X}^{\top}\boldsymbol{X}\otimes\left(G_{t}G_{t}^{\top}\right)^{1/2}\right)=\lambda_{\text{min}}\left(\boldsymbol{X}^{\top}\boldsymbol{X}\right)\lambda_{\text{min}}\left(\left(G_{t}G_{t}^{\top}\right)^{1/2}\right)\\
 & =\lambda_{\text{min}}\left(\boldsymbol{X}^{\top}\boldsymbol{X}\right)\lambda_{\text{min}}\left(G_{t}G_{t}^{\top}\right)^{1/2}.
\end{align*}

\[
\frac{\alpha}{\beta}=\frac{\tilde{\alpha}}{\tilde{\beta}}\sqrt{\frac{\lambda_{\text{min}}\left(G_{t}G_{t}^{\top}\right)}{\lambda_{\text{max}}\left(G_{t}G_{t}^{\top}\right)}}<\frac{\tilde{\alpha}}{\tilde{\beta}}.
\]

This completes the proof.
\end{proof}

\section{Mechanistic Analysis: Parameter and Gradient Norms}
\label{app:mechanistic_norms}

\begin{figure*}[htpb]
    \centering
    \begin{subfigure}{0.24\textwidth}
        \includegraphics[width=\linewidth]{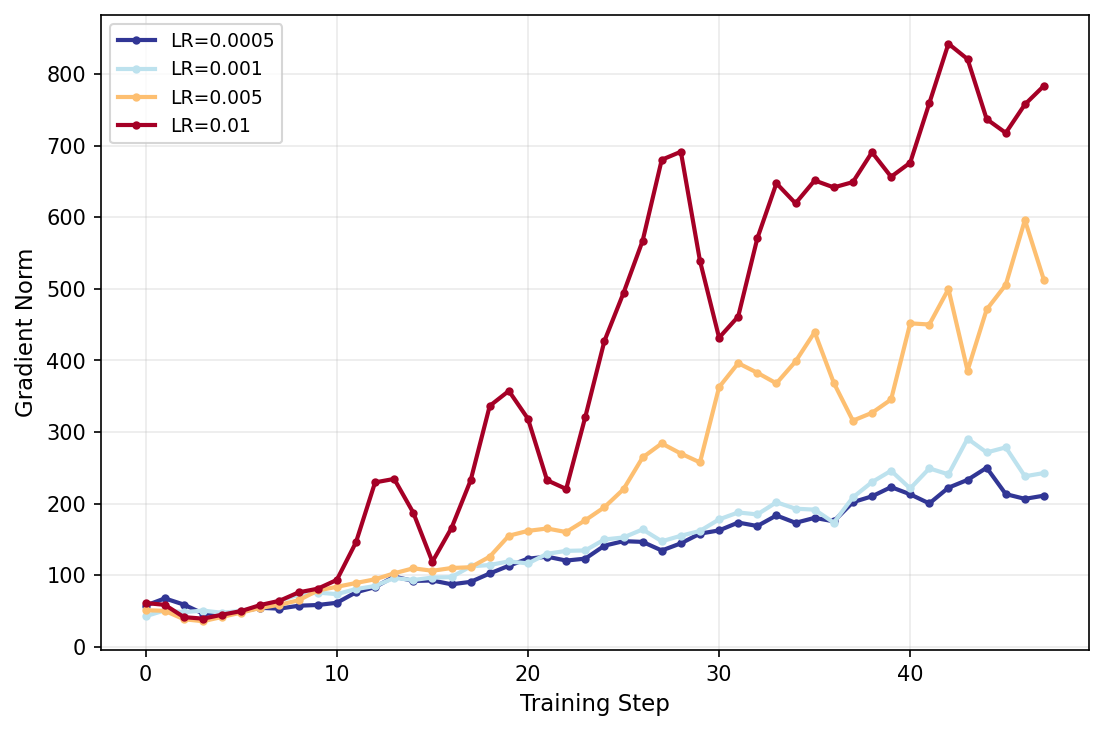}
        \caption{Muon (Grad)}
    \end{subfigure}\hfill
    \begin{subfigure}{0.24\textwidth}
        \includegraphics[width=\linewidth]{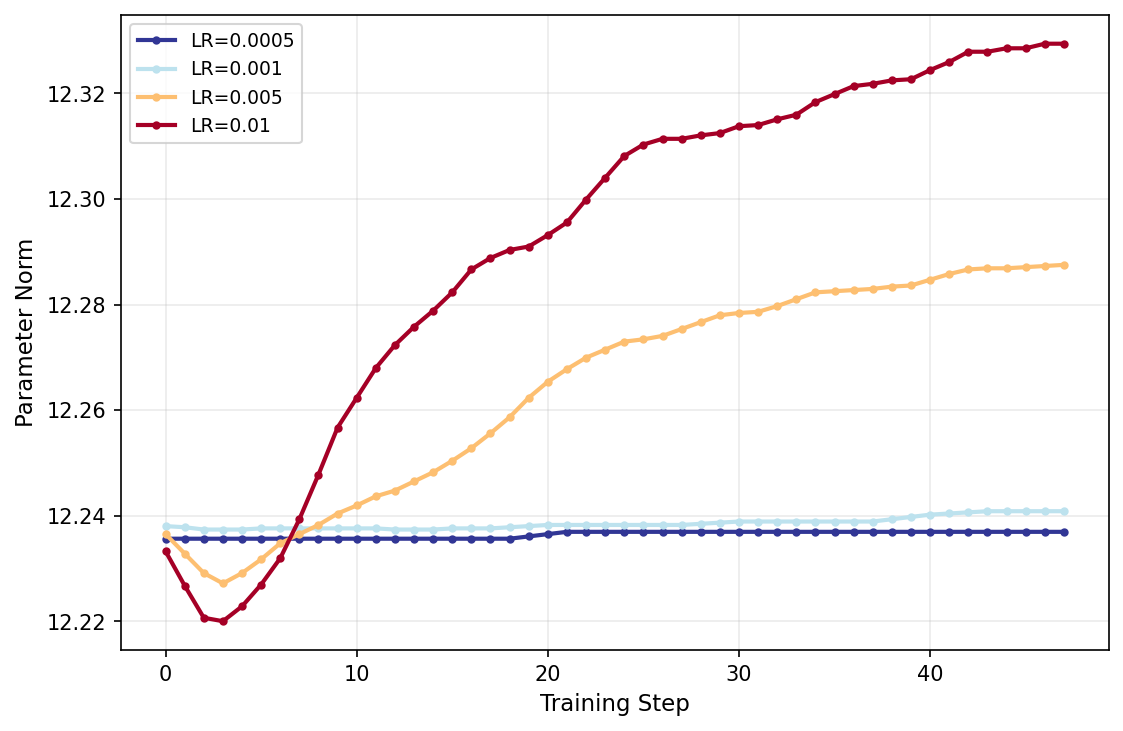}
        \caption{Muon (Param)}
    \end{subfigure}\hfill
    \begin{subfigure}{0.24\textwidth}
        \includegraphics[width=\linewidth]{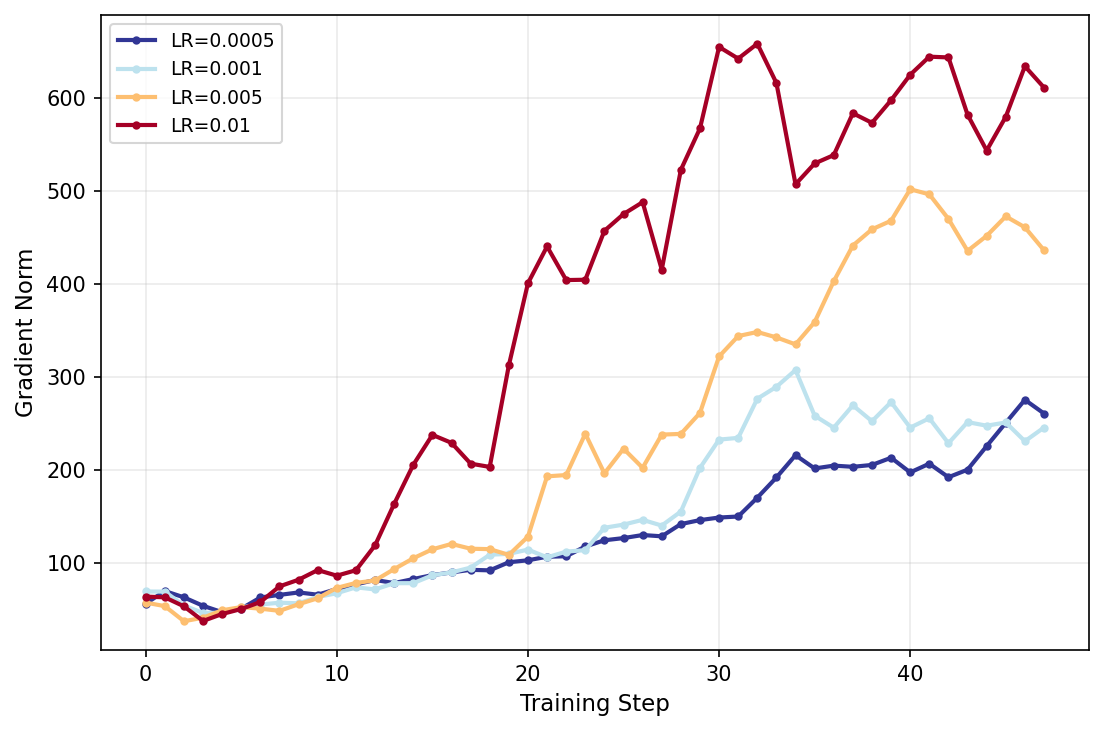}
        \caption{Muon + Mom (Grad)}
    \end{subfigure}\hfill
    \begin{subfigure}{0.24\textwidth}
        \includegraphics[width=\linewidth]{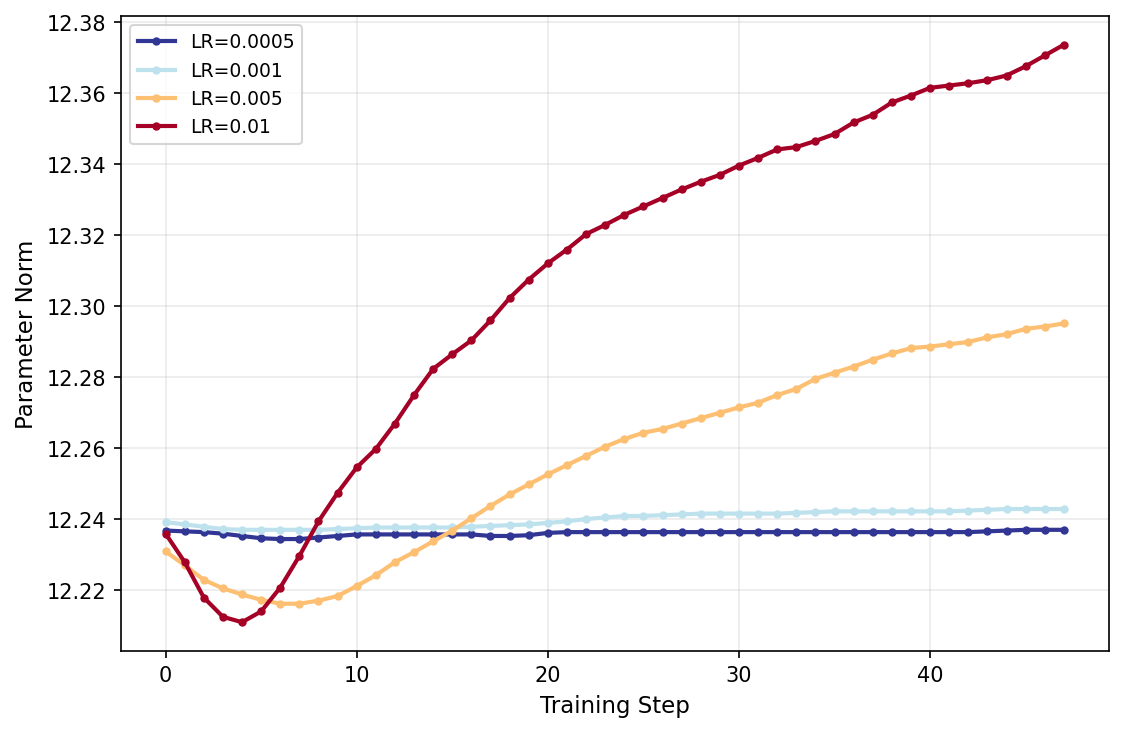}
        \caption{Muon + Mom (Param)}
    \end{subfigure}
    
    \vspace{0.4cm}
    \begin{subfigure}{0.24\textwidth}
        \includegraphics[width=\linewidth]{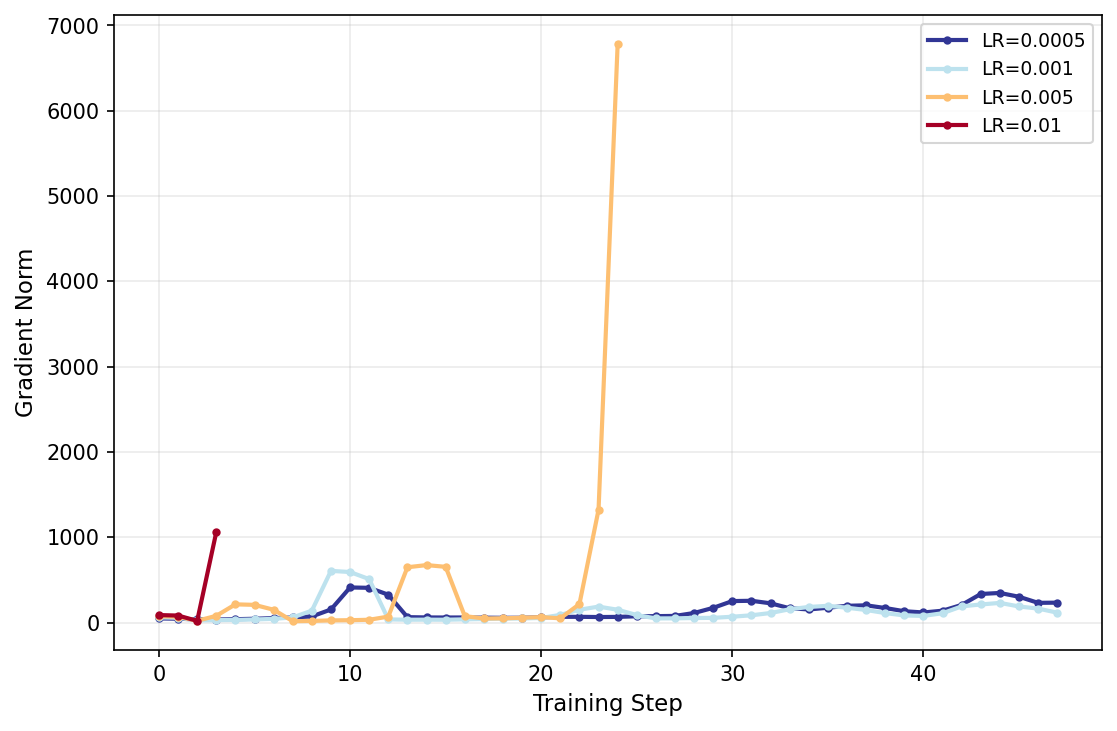}
        \caption{SGD (Grad)}
    \end{subfigure}\hfill
    \begin{subfigure}{0.24\textwidth}
        \includegraphics[width=\linewidth]{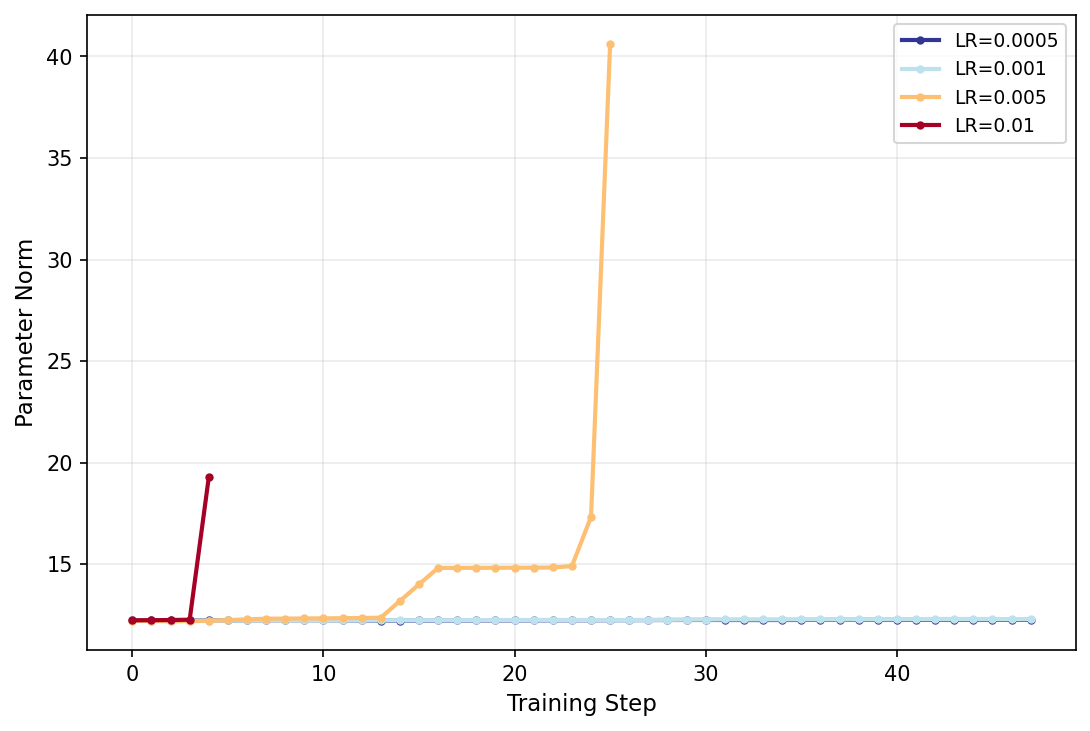}
        \caption{SGD (Param)}
    \end{subfigure}\hfill
    \begin{subfigure}{0.24\textwidth}
        \includegraphics[width=\linewidth]{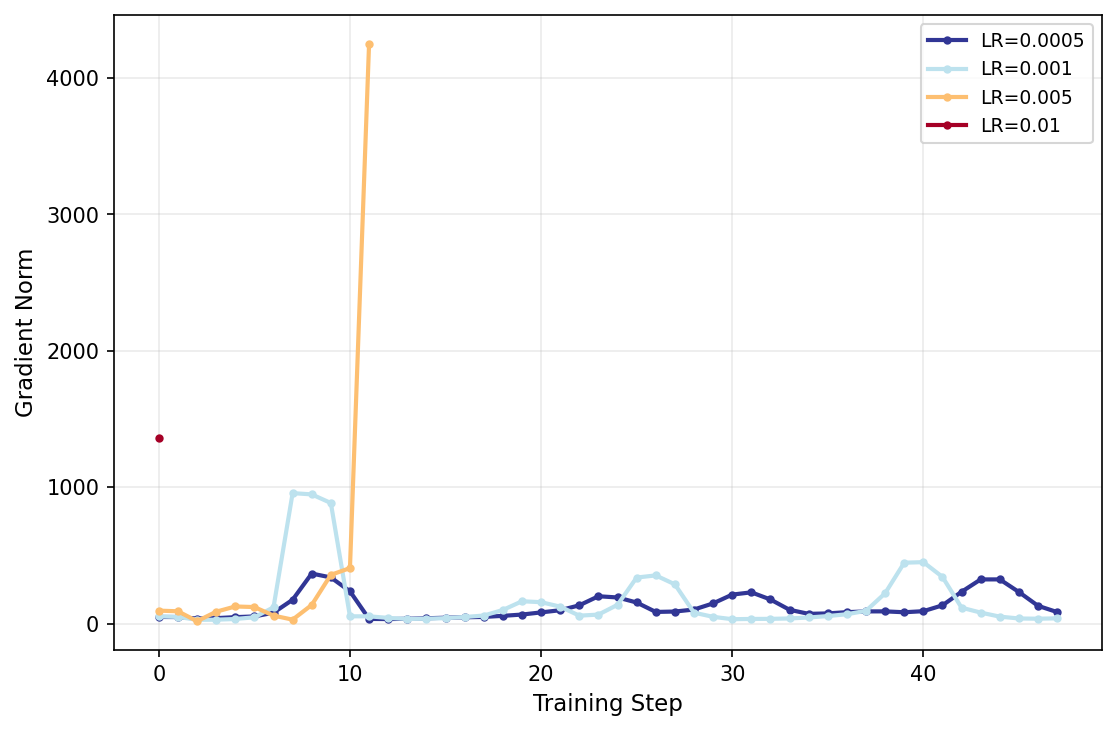}
        \caption{SGD + Mom (Grad)}
    \end{subfigure}\hfill
    \begin{subfigure}{0.24\textwidth}
        \includegraphics[width=\linewidth]{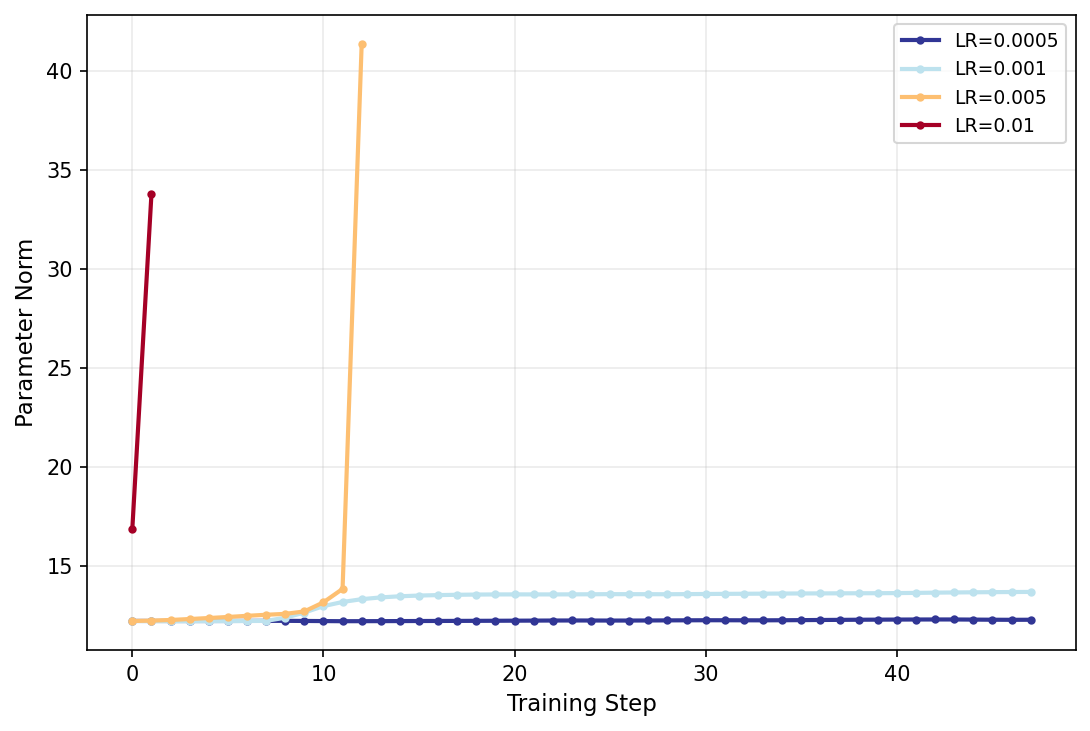}
        \caption{SGD + Mom (Param)}
    \end{subfigure}
    
    \caption{Early training dynamics (first 50 steps) of Gradient and Parameter norms. The top row demonstrates Muon's controlled, steady norm growth via spectral flattening. The bottom row reveals SGD's catastrophic parameter explosion under high learning rates, leading to immediate divergence.}
    \label{fig:exp1_norms}
\end{figure*}

To better understand the source of SGD's instability and to probe the effect predicted by Theorem~\ref{thm:eta_Muon}, we track both gradient norm and parameter norm during the first 50 training steps, as shown in Figure~\ref{fig:exp1_norms}. At a high learning rate, SGD shows a rapid increase in both quantities, indicating that the optimization trajectory quickly leaves a stable regime. This behavior is consistent with the presence of a dominant singular direction in the update, which can accumulate aggressively when the gradient is applied directly. Muon behaves quite differently. By orthogonalizing the momentum buffer through Newton--Schulz iterations, it reduces the influence of any single dominating direction in the update. As shown in the top row of Figure~\ref{fig:exp1_norms}, Muon's parameter norm grows much more gradually, and its gradient norm remains well controlled even at learning rates that cause SGD to diverge. Taken together, these results suggest that Muon's spectral flattening effect makes large learning rates substantially more tolerable by reducing the impact of extreme singular directions in the update.

\section{Validation Accuracy Threshold Analysis}
\label{sec:threshold_analysis}

\begin{figure*}[htpb]
    \centering
    \begin{subfigure}{0.48\textwidth}
        \includegraphics[width=\linewidth]{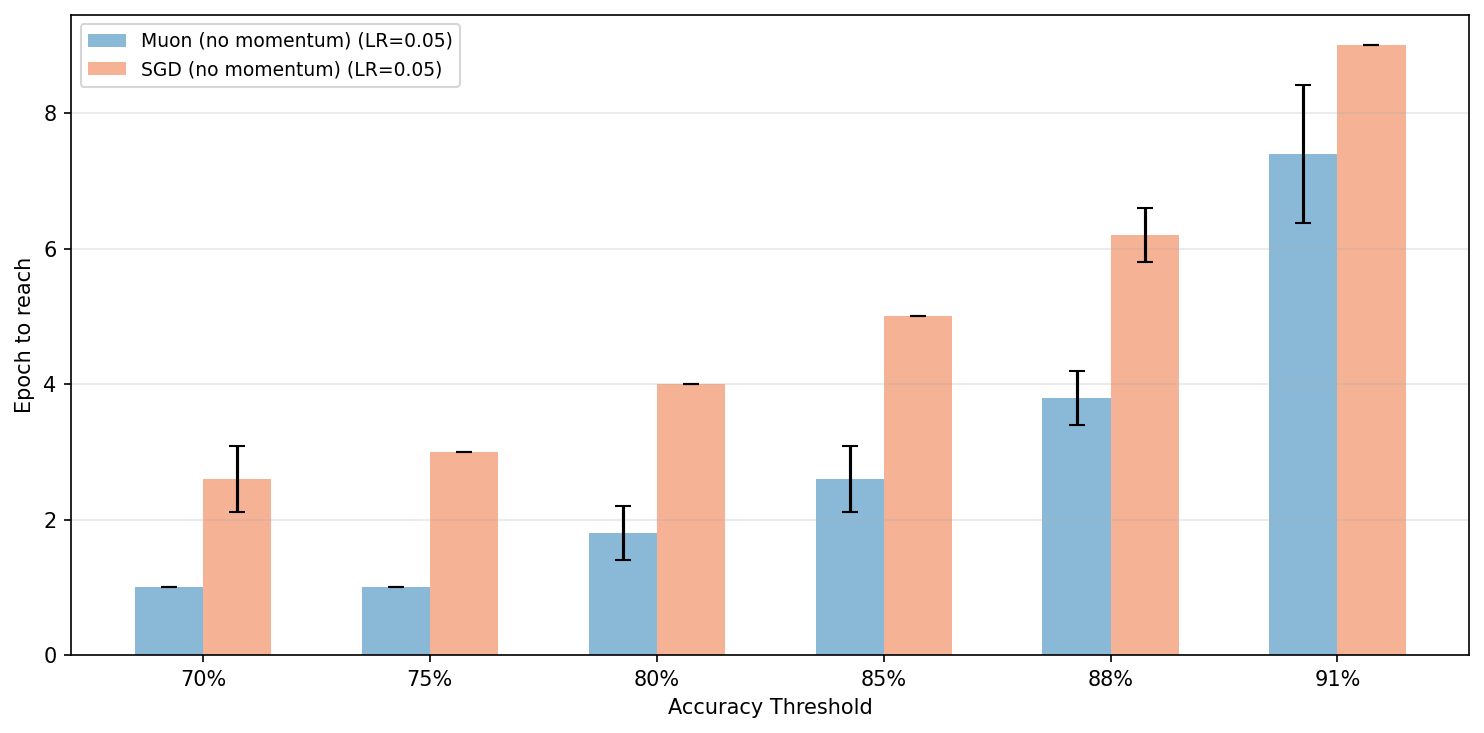}
        \caption{No Momentum}
    \end{subfigure}\hfill
    \begin{subfigure}{0.48\textwidth}
        \includegraphics[width=\linewidth]{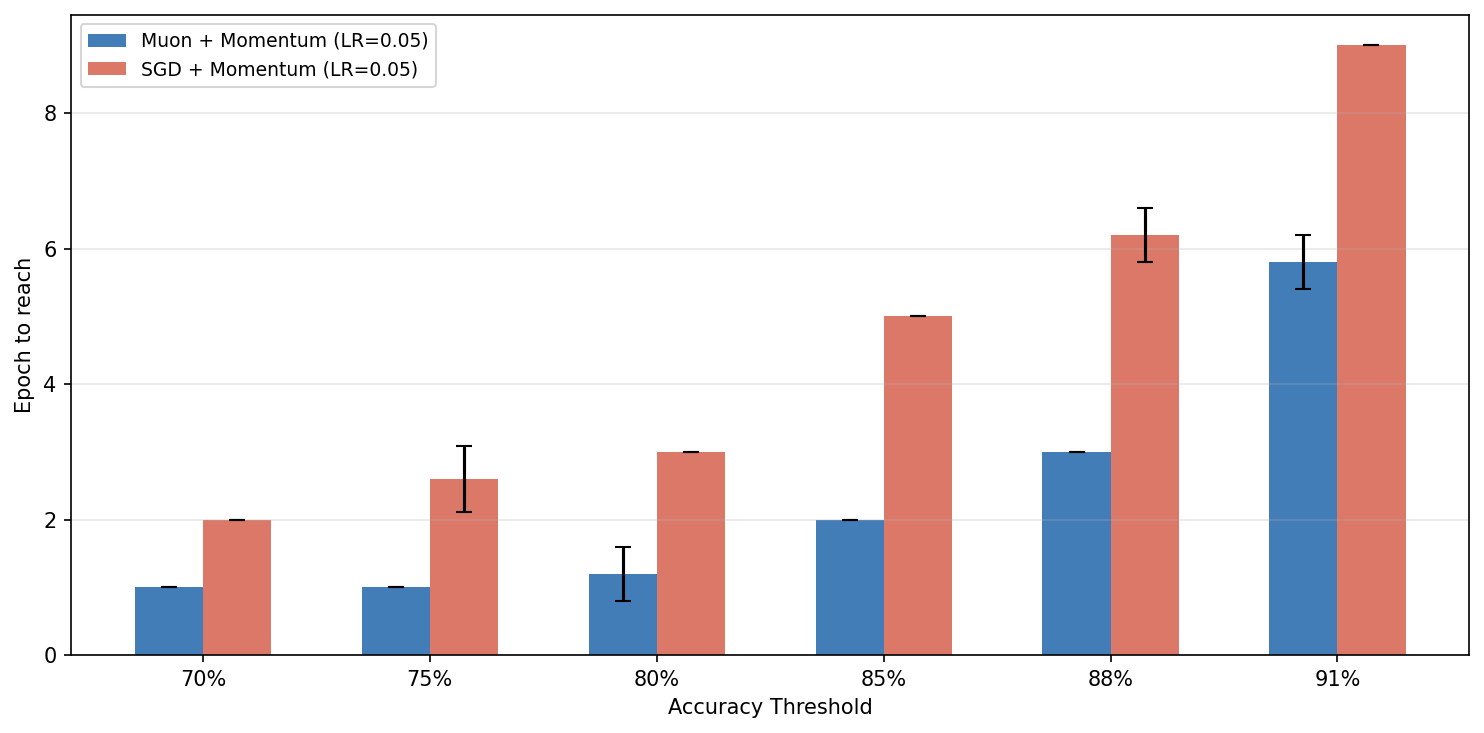}
        \caption{With Momentum}
    \end{subfigure}
    \caption{Epochs required to reach specific validation accuracy thresholds. Lower bars indicate faster convergence. Muon reaches all critical milestones significantly earlier than SGD.}
    \label{fig:exp2_threshold}
\end{figure*}

To complement the accuracy and loss curves presented in the main text, we provide a per-threshold breakdown of convergence speed. Figure~\ref{fig:exp2_threshold} reports the number of epochs each optimizer requires to reach specific validation accuracy milestones ranging from $70\%$ to $91\%$. Across all thresholds and both momentum settings, Muon consistently reaches every milestone several epochs ahead of SGD. The gap widens at higher accuracy targets, indicating that Muon's advantage is not confined to early training but persists---and even grows---as optimization approaches more challenging regions of the loss landscape. This pattern is consistent with the improved effective convergence factor predicted by Theorem~\ref{thm:comparision}: spectral preconditioning yields compounding gains over successive iterations.

\section{Normalization Principle in Transformer Architectures}
\label{app:transformer}


\begin{figure}
    \centering
    \includegraphics[width=\textwidth]{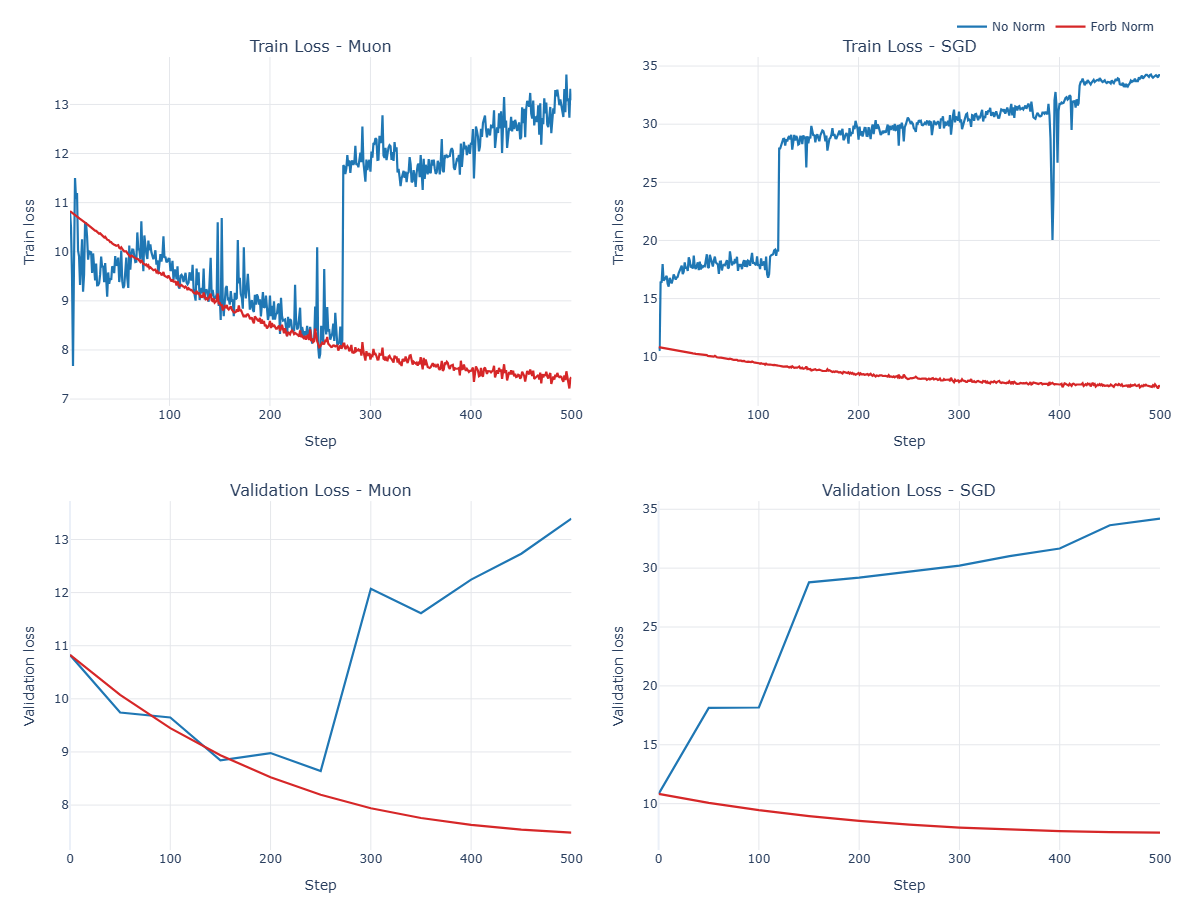}
    \caption{Training behavior of a GPT-2-style Transformer trained with FrobNorm at a high learning rate.}
    \label{fig:norm_transformer}
\end{figure}

To further test whether the proposed normalization principle generalizes beyond convolutional networks, we apply FrobNorm, introduced in
Section~\ref{sec:normalization}, to a GPT-2-style Transformer architecture. In particular, we replace the normalization layers placed before and after the $12$-layer Transformer stack with FrobNorm. We compare FrobNorm against a no-normalization baseline in order to isolate the effect of spectral normalization on high-learning-rate stability.

We pretrain the model from scratch on $50$ million tokens. To stress-test the
stability of the normalization method, we use an unusually large learning rate
of $0.2$, which is typically unstable for this architecture. Figure~\ref{fig:norm_transformer} shows that the same principle also transfers to Transformer architectures. At the large learning rate of $0.2$, the model without normalization is unstable for both Muon and SGD. With Muon, the
no-normalization baseline initially decreases its training and validation loss, but becomes unstable after roughly $270$ steps, after which both losses increase sharply. In contrast, the model with FrobNorm continues to reduce both training and validation loss throughout training.

The difference is even more pronounced for SGD. Without normalization, the
training loss rapidly increases and the validation loss diverges, indicating that this learning rate is far outside the stable regime. By contrast, FrobNorm keeps training stable and steadily decreases both training and validation loss.

These results suggest that controlling the dominant spectral scale of layer
inputs is a general normalization principle that can improve high-learning-rate stability across multiple architecture families.

\section{Convergence Rate with Best Learning Rates}
\label{app:convergence_best_lr}

In the main text (Section~\ref{sec:exp_convergence}), we compare Muon and SGD under the same learning rate to isolate the intrinsic efficiency of the optimizers. Here we complement that analysis by assigning each optimizer its best-performing learning rate as established in Section~\ref{sec:exp_lr}: $\eta_{\mathrm{conv}} = 0.1$ for Muon and $\eta_{\mathrm{conv}} = 0.01$ for SGD. All other settings (BN, linear scheduler, 5 runs) remain identical. This setup reflects realistic practice, where practitioners tune the learning rate to each optimizer's strength, and allows us to evaluate whether Muon's advantage persists even when SGD is given its own optimal learning rate.

\begin{figure*}[htpb]
    \centering
    \begin{subfigure}{0.48\textwidth}
        \includegraphics[width=\linewidth]{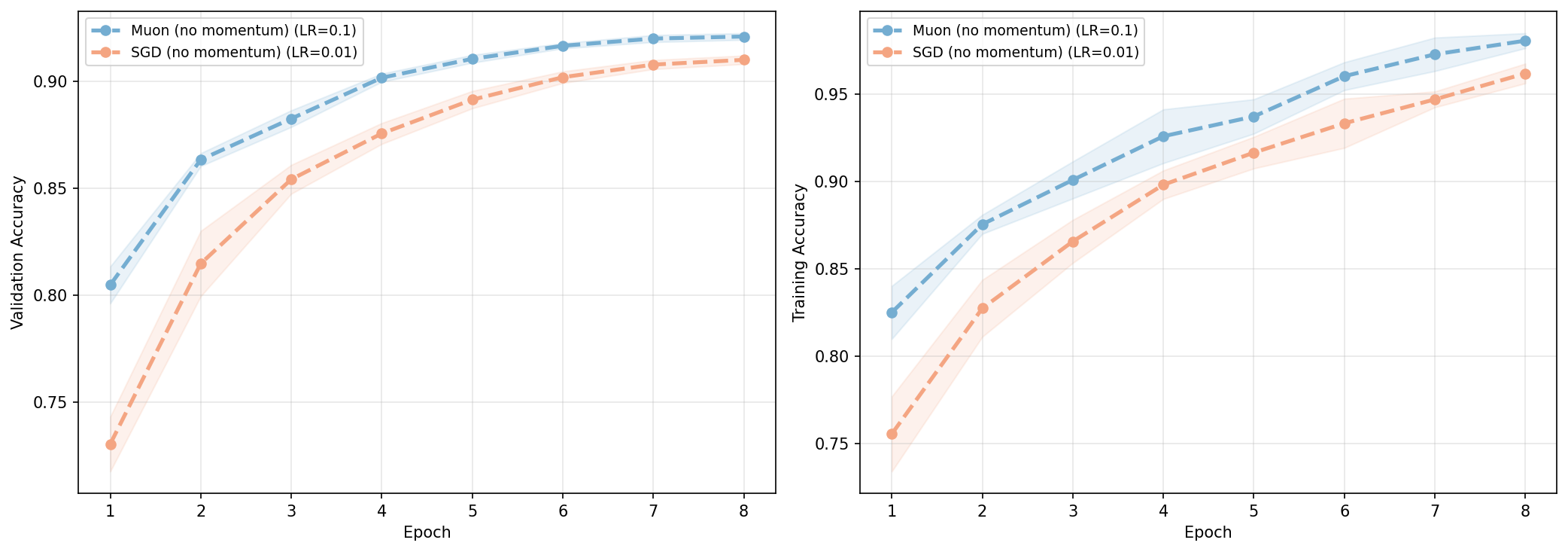}
        \caption{No Momentum}
    \end{subfigure}\hfill
    \begin{subfigure}{0.48\textwidth}
        \includegraphics[width=\linewidth]{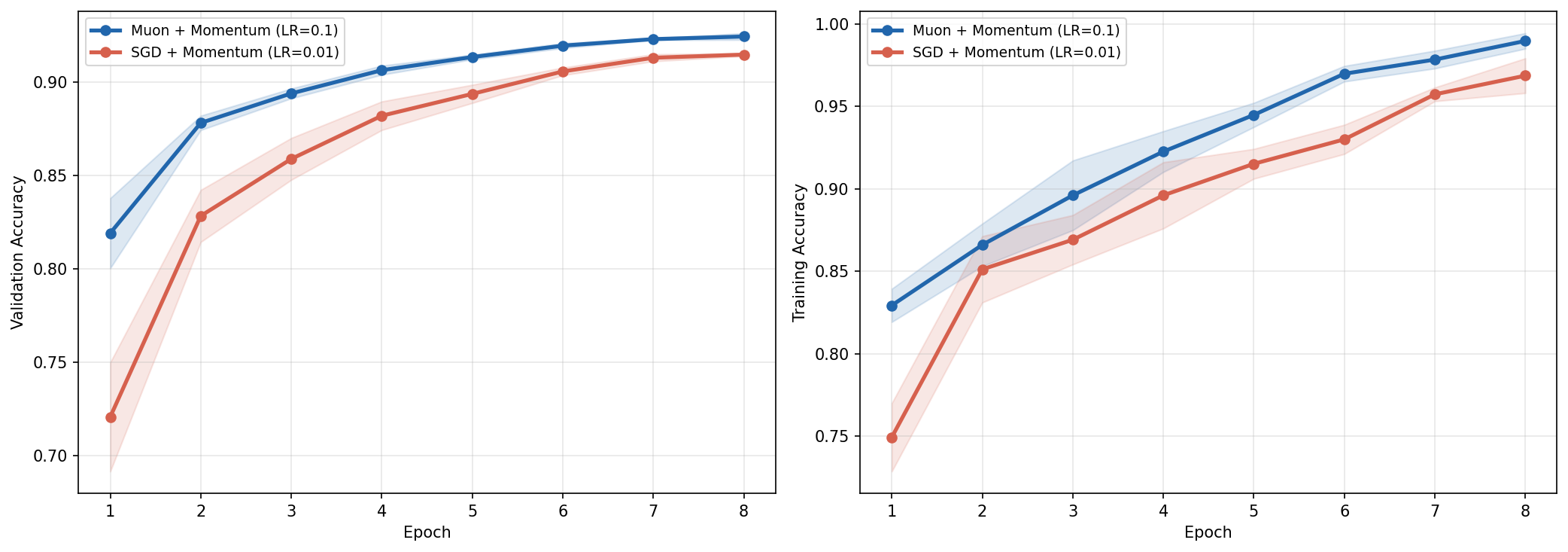}
        \caption{With Momentum}
    \end{subfigure}
    \caption{Training and Validation Accuracy vs. Epoch with best learning rates (Muon at $0.1$, SGD at $0.01$).}
    \label{fig:exp2_best_lr_acc}
\end{figure*}

\begin{figure*}[htpb]
    \centering
    \begin{subfigure}{0.48\textwidth}
        \includegraphics[width=\linewidth]{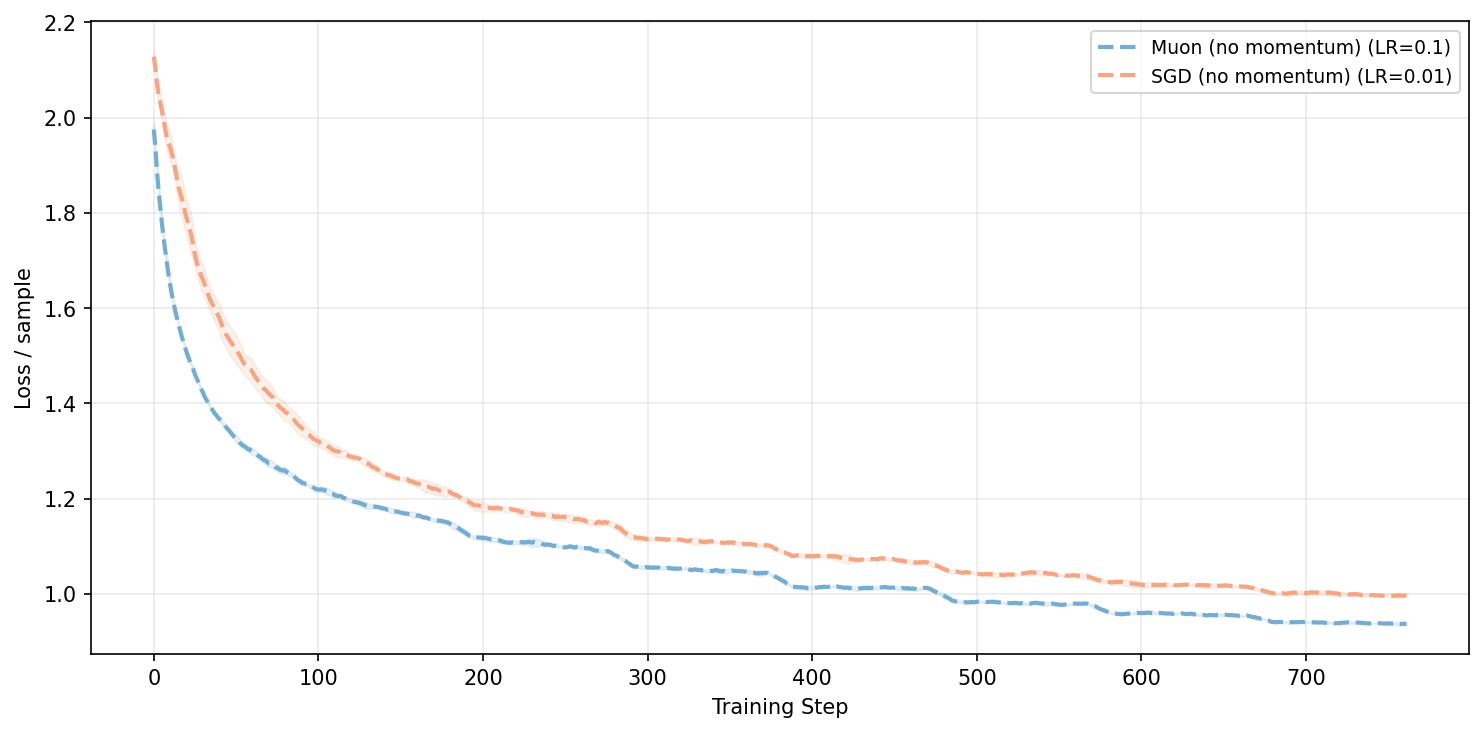}
        \caption{No Momentum}
    \end{subfigure}\hfill
    \begin{subfigure}{0.48\textwidth}
        \includegraphics[width=\linewidth]{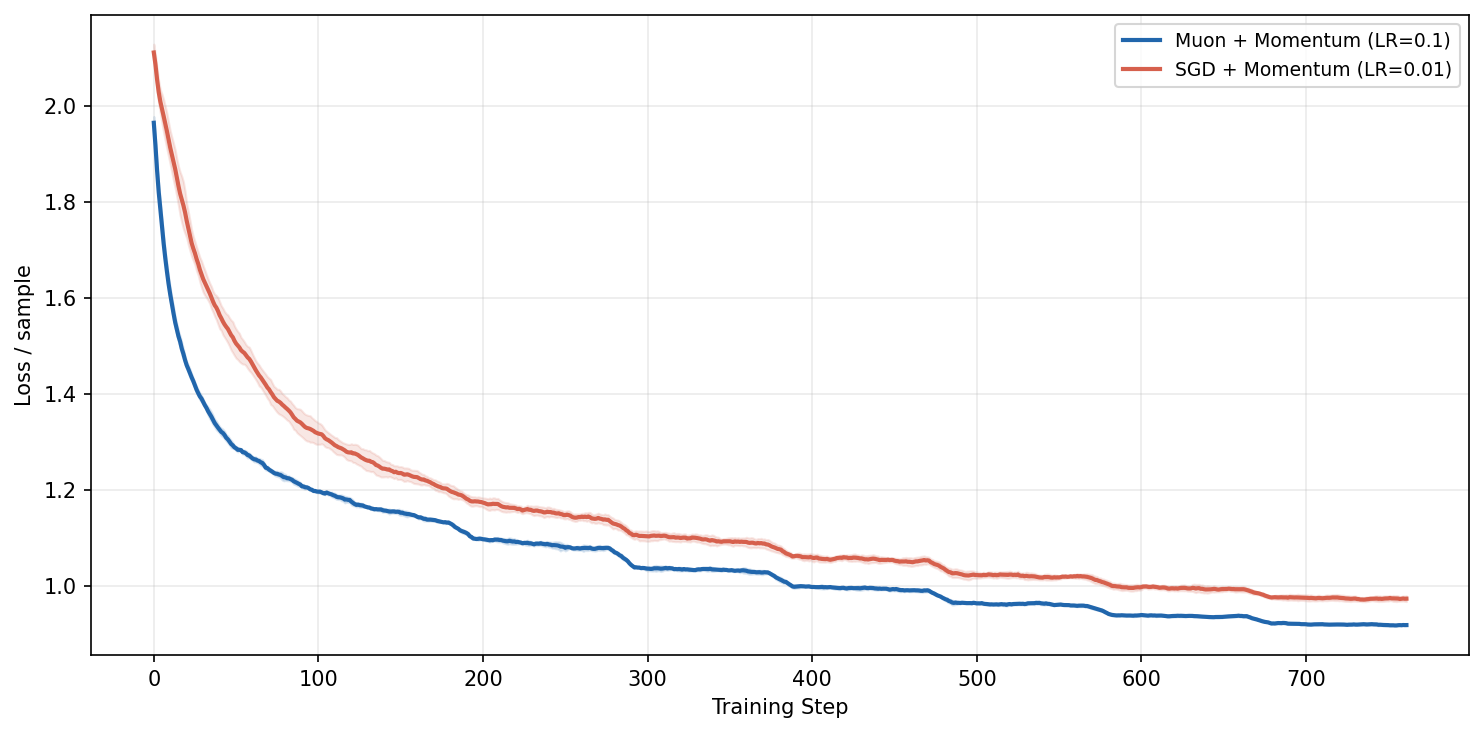}
        \caption{With Momentum}
    \end{subfigure}
    \caption{Training Loss vs. Step with best learning rates (Muon at $0.1$, SGD at $0.01$).}
    \label{fig:exp2_best_lr_loss}
\end{figure*}

\begin{figure*}[htpb]
    \centering
    \begin{subfigure}{0.48\textwidth}
        \includegraphics[width=\linewidth]{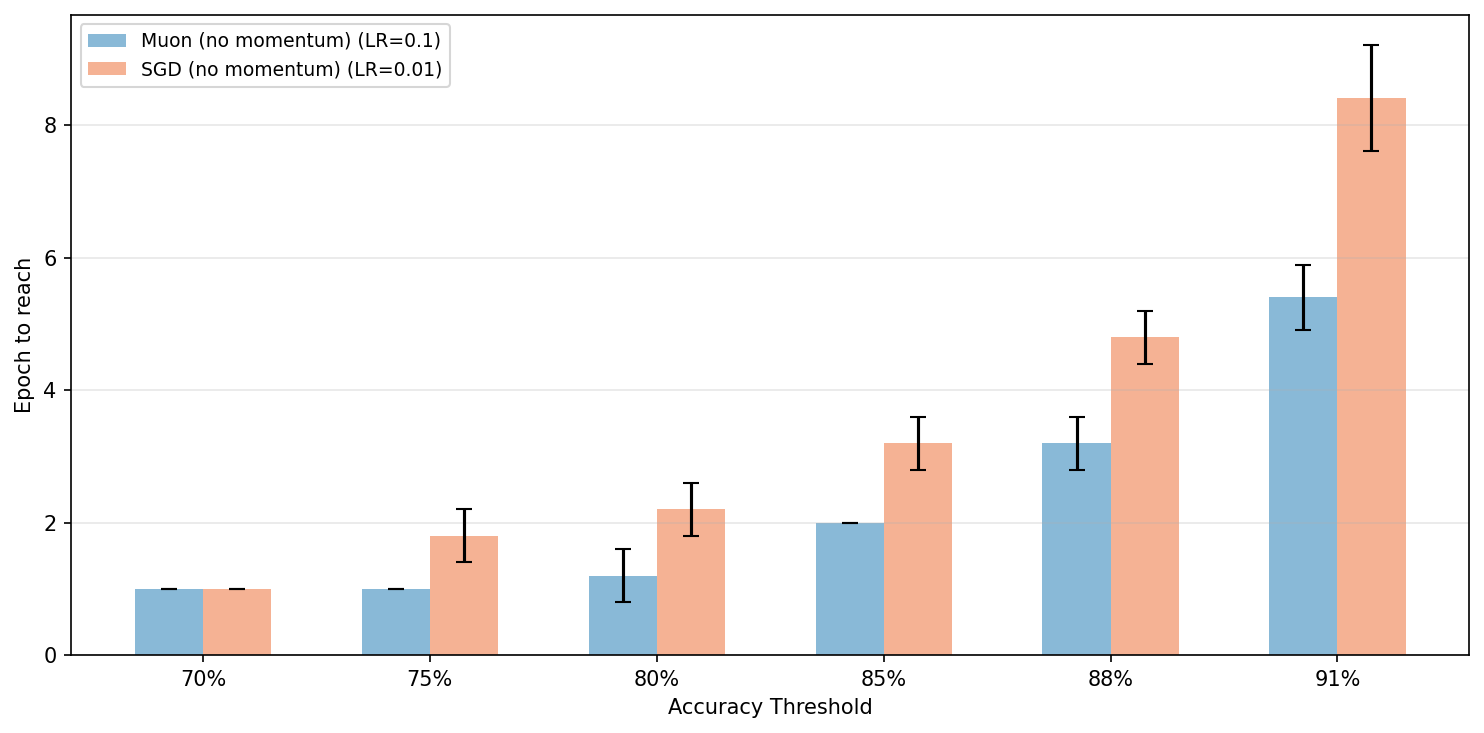}
        \caption{No Momentum}
    \end{subfigure}\hfill
    \begin{subfigure}{0.48\textwidth}
        \includegraphics[width=\linewidth]{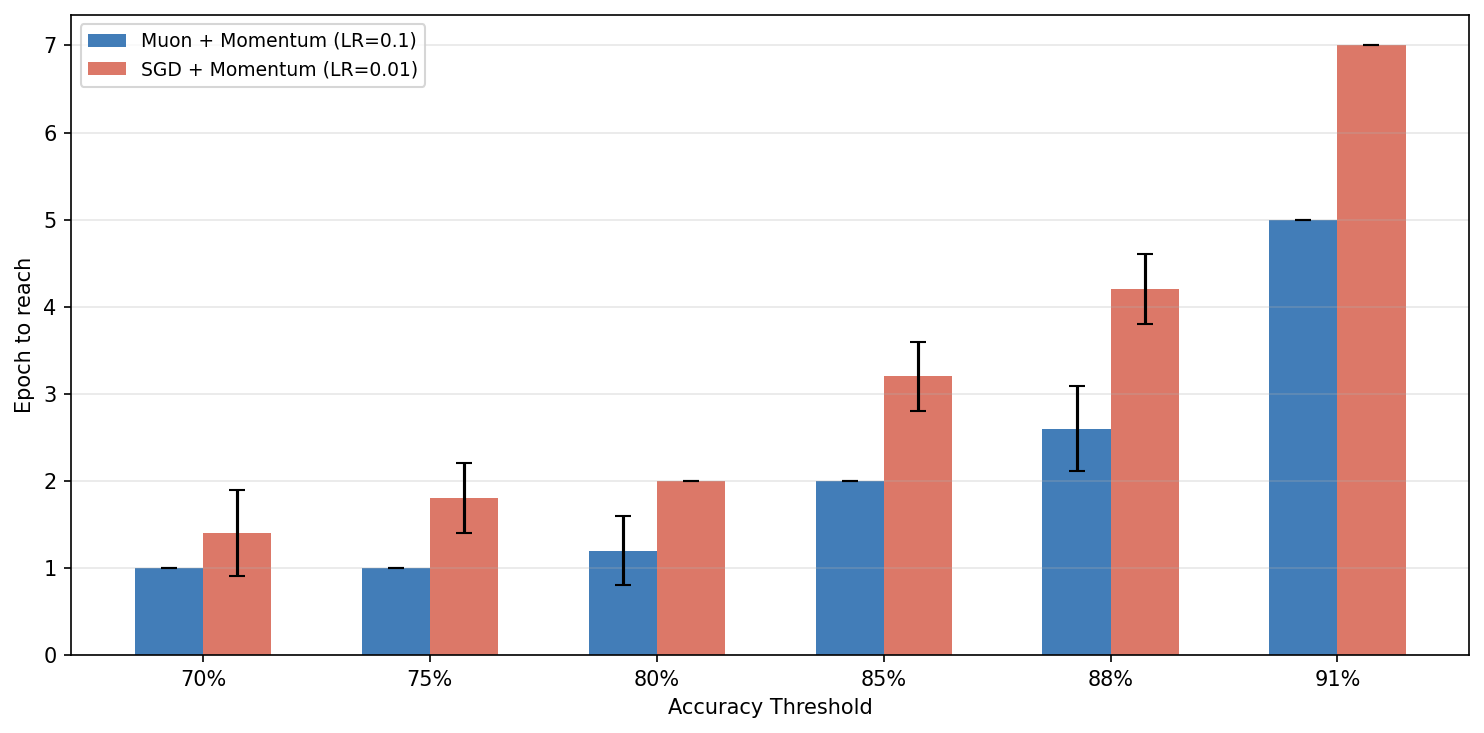}
        \caption{With Momentum}
    \end{subfigure}
    \caption{Epochs to reach validation accuracy thresholds with best learning rates (Muon at $0.1$, SGD at $0.01$).}
    \label{fig:exp2_best_lr_threshold}
\end{figure*}

\begin{figure*}[htpb]
    \centering
    \begin{subfigure}{0.48\textwidth}
        \includegraphics[width=\linewidth]{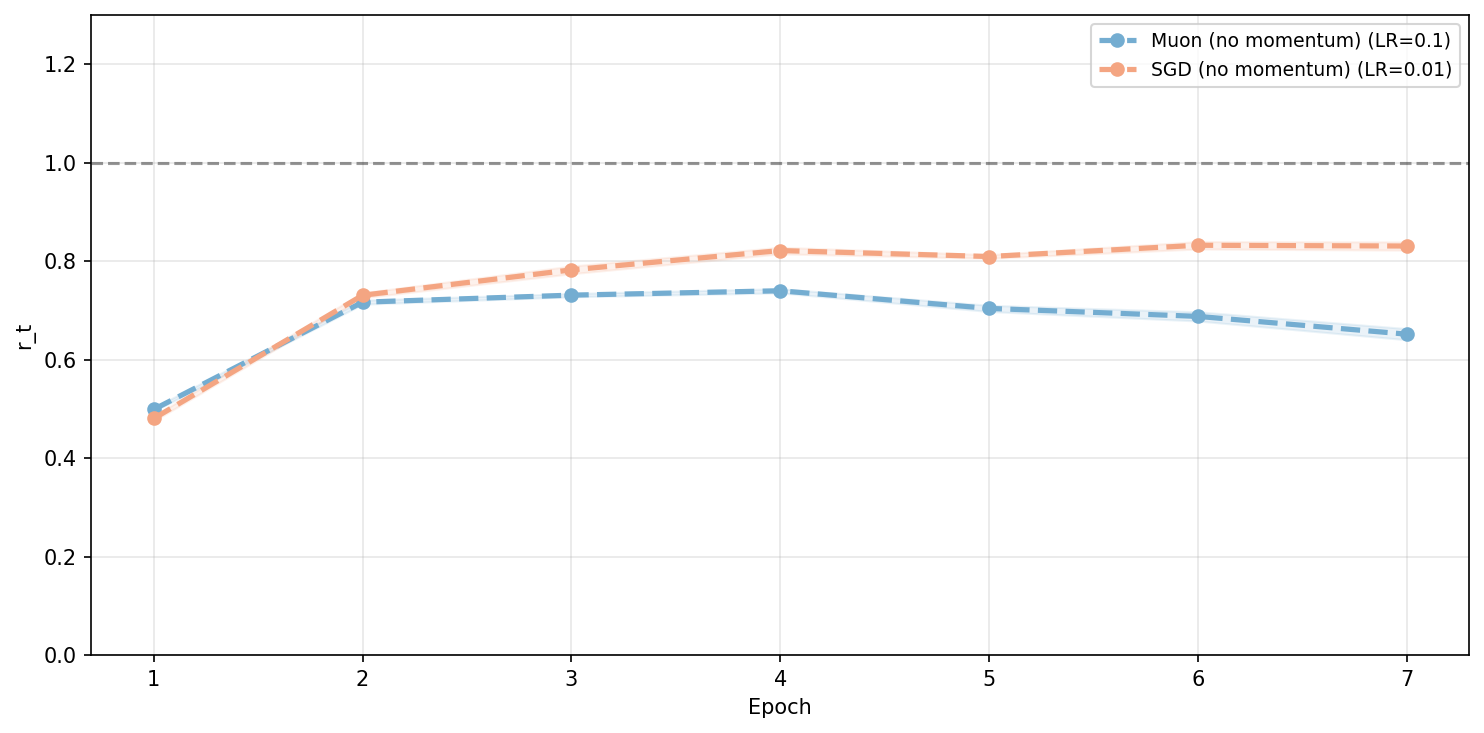}
        \caption{No Momentum}
    \end{subfigure}\hfill
    \begin{subfigure}{0.48\textwidth}
        \includegraphics[width=\linewidth]{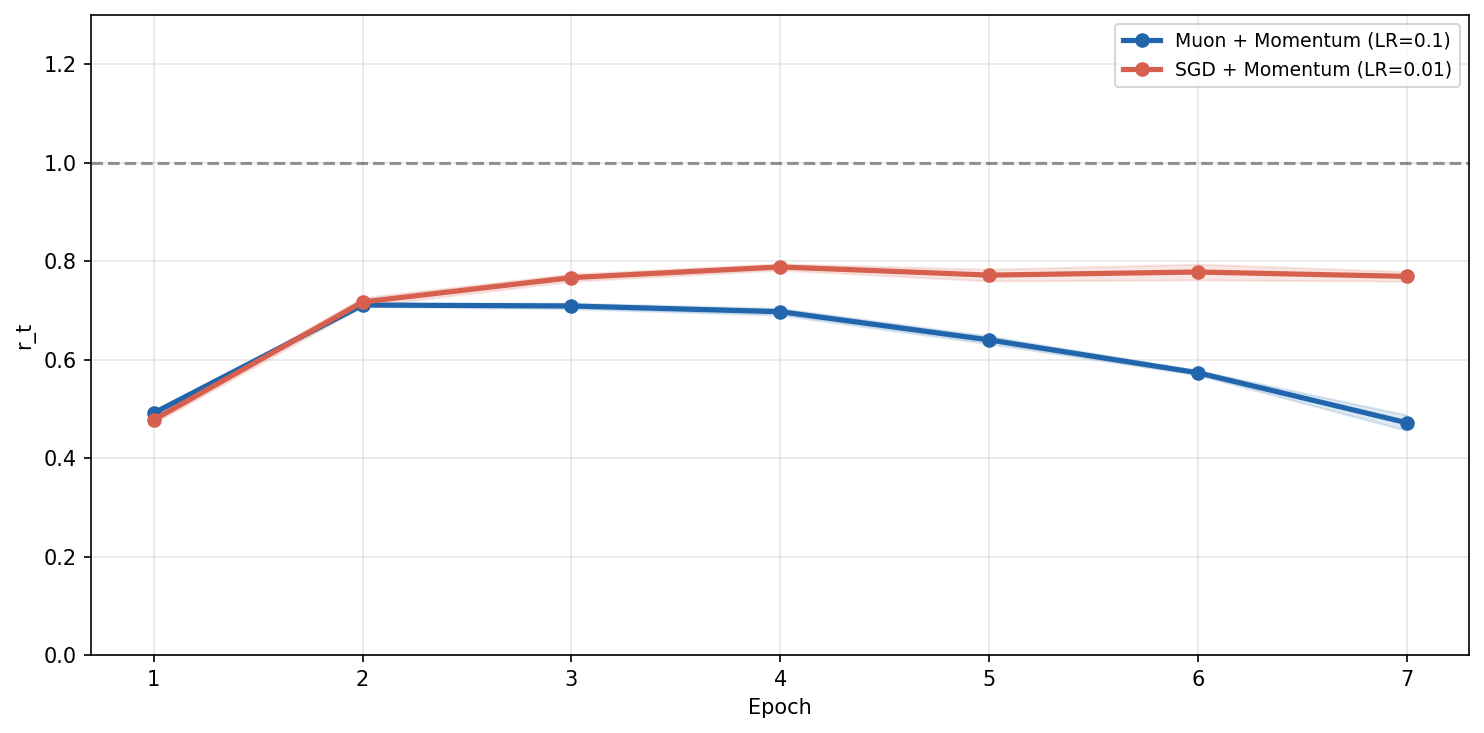}
        \caption{With Momentum}
    \end{subfigure}
    \caption{Empirical Convergence Ratio ($r_t$) with best learning rates (Muon at $0.1$, SGD at $0.01$). Muon maintains a consistently lower $r_t$ throughout training.}
    \label{fig:exp2_best_lr_rt}
\end{figure*}

The results confirm that Muon's convergence advantage is not merely an artifact of using the same step size for both optimizers. Even when SGD is given its own best learning rate of $0.01$, Muon at $0.1$ reaches higher accuracy sooner (Figure~\ref{fig:exp2_best_lr_acc}), achieves steeper loss descent (Figure~\ref{fig:exp2_best_lr_loss}), arrives at all validation milestones earlier (Figure~\ref{fig:exp2_best_lr_threshold}), and maintains a consistently lower convergence ratio $r_t$ (Figure~\ref{fig:exp2_best_lr_rt}). This demonstrates that the convergence acceleration predicted by Theorem~\ref{thm:comparision} translates into a practical advantage under realistic tuning conditions.

\section{Compute Resources for Experiments}
\label{app:compute}

All CNN experiments on CIFAR-10 (Sections~\ref{sec:exp_lr}--\ref{sec:exp_convergence} and Appendix~\ref{app:convergence_best_lr}) were run on a single NVIDIA A100 40\,GB GPU. The Transformer pretraining experiment (Appendix~\ref{app:transformer}) was run on 4 NVIDIA A100 40\,GB GPUs. Every configuration was repeated with 5 random seeds $\{0, 1, 2, 3, 4\}$.

\section{Limitations and Broader Impacts}

\textbf{Limitations.}
Our analysis assumes the deterministic full-batch setting and the exact polar factor $UV^\top$, whereas practical Muon uses stochastic gradients, finite Newton--Schulz iterations, and momentum. The convergence rate comparison relies on a Kronecker-factored Hessian approximation that may be less accurate outside the convolutional layers studied here. Our experiments are confined to CIFAR-10 with CifarNet; validating the quantitative predictions on larger-scale architectures and datasets remains an open direction. The preconditioning framework applies to matrix-shaped parameters and does not address scalar parameter groups (biases, normalization scales), which still require a separate optimizer.

\textbf{Broader Impacts.}
This work is primarily theoretical and methodological, aimed at understanding the mechanism behind Muon's empirical success. Faster convergence and tolerance of larger learning rates can reduce the computational cost and energy consumption of training, though these savings could also be reinvested into training larger models. We see no unique negative societal consequences from this contribution beyond those common to general advances in optimization for deep learning.



\end{document}